\documentclass{article}

\usepackage{url}
\usepackage{amsfonts}       %
\usepackage{amsmath,amssymb}
\usepackage{amsthm}
\usepackage{nicefrac}       %
\usepackage{microtype}      %
\usepackage{url}
\usepackage{array,booktabs,makecell}
\usepackage{natbib}
\usepackage{hyperref}
\usepackage[capitalize]{cleveref}
\usepackage{autonum}
\usepackage{mathtools}
\usepackage[colorinlistoftodos]{todonotes}
\usepackage{enumitem}
\usepackage[parfill]{parskip}
\usepackage{bm}
\usepackage{mathrsfs}
\usepackage{subfigure} 
\usepackage{dblfloatfix}
\usepackage{multirow}
\usepackage{wrapfig}
\usepackage{multicol}
\usepackage{caption}
\usepackage[scr=esstix]{mathalfa} %
\usepackage{mathtools}
\usepackage{float}
\graphicspath{{figures/}}

\usepackage{algorithmic}
\usepackage{algorithm}
\usepackage{changepage}

\newtheorem{theorem}{Theorem}
\newtheorem{corollary}{Corollary}[theorem]
\newtheorem{lemma}{Lemma}

\newtheorem{definition}{Definition}

\DeclareMathOperator*{\argmin}{arg\,min}

\newcounter{commentCounter}
\newif\iftrvar
\trvarfalse
\iftrvar
\newcommand{\sw}[1]{{\small \color{red} \refstepcounter{commentCounter}\textsf{[SW]$_{\arabic{commentCounter}}$:{#1}}}}
\newcommand{\am}[1]{{\small \color{blue} \refstepcounter{commentCounter}\textsf{[AM]$_{\arabic{commentCounter}}$:{#1}}}}
\newcommand{\ms}[1]{{\small \color{orange} \refstepcounter{commentCounter}\textsf{[MS]$_{\arabic{commentCounter}}$:{#1}}}}
\newcommand{\tr}[1]{{\small \color{darkgreen} \refstepcounter{commentCounter}\textsf{[TR]$_{\arabic{commentCounter}}$:{#1}}}}
\newcommand{\mingfei}[1]{{\small \color{purple} \refstepcounter{commentCounter}\textsf{[Mingfei]$_{\arabic{commentCounter}}$:{#1}}}}

\else
\newcommand{\am}[1]{}
\newcommand{\sw}[1]{}
\newcommand{\ms}[1]{}
\newcommand{\tr}[1]{}
\newcommand{\mingfei}[1]{}

\fi

\definecolor{lightgray}{rgb}{.9,.9,.9}
\definecolor{darkgray}{rgb}{.4,.4,.4}
\definecolor{purple}{rgb}{0.65, 0.12, 0.82}
\definecolor{darkgreen}{rgb}{0, 0.365, 0}
\definecolor{orange}{rgb}{1,0.5,0}
\definecolor{deep}{rgb}{0.13, 0.46, 0.7}

\usepackage[accepted]{icml2022}

\icmltitlerunning{Generalization in Cooperative Multi-Agent Systems}

\begin{document}

\setlength{\abovedisplayskip}{4pt}
\setlength{\belowdisplayskip}{4pt}

\twocolumn[
\icmltitle{Generalization in Cooperative Multi-Agent Systems}
\begin{icmlauthorlist}
\icmlauthor{Anuj Mahajan}{ox}
\icmlauthor{Mikayel Samvelyan}{ucl,fb}
\icmlauthor{Tarun Gupta}{ox}
\icmlauthor{Benjamin Ellis}{ox}
\icmlauthor{Mingfei Sun}{ox}
\icmlauthor{Tim Rocktäschel}{ucl,fb}
\icmlauthor{Shimon Whiteson}{ox}
\end{icmlauthorlist}

\icmlaffiliation{ox}{University of Oxford}
\icmlaffiliation{ucl}{University College London}
\icmlaffiliation{fb}{Meta AI}

\icmlcorrespondingauthor{Anuj Mahajan}{anuj.mahajan@cs.ox.ac.uk}

\icmlkeywords{Machine Learning, Reinforcement Learning, Generalization, Multi agent systems}

\vskip 0.3in
]

\printAffiliationsAndNotice{}  %

\begin{abstract}

Collective intelligence is a fundamental trait shared by several species of living organisms. It has allowed them to thrive in the diverse environmental conditions that exist on our planet. From simple organisations in an ant colony to complex systems in human groups, collective intelligence is vital for solving complex survival tasks.
As is commonly observed, such natural systems are flexible to changes in their structure.
Specifically, they exhibit a high degree of generalization when the abilities or the total number of agents changes within a system.
We term this phenomenon as \textit{Combinatorial Generalization} (CG). 
CG is a highly desirable trait for autonomous systems as it can increase their utility and deployability across a wide range of applications. 
While recent works addressing specific aspects of CG have shown impressive results on complex domains, they provide no performance guarantees when generalizing towards novel situations.
In this work, we shed light on the theoretical underpinnings of CG for cooperative multi-agent systems (MAS). 
Specifically, we study generalization bounds under a linear dependence of the underlying dynamics on the agent capabilities, which can be seen as a generalization of Successor Features to MAS. We then extend the results first for Lipschitz and then arbitrary dependence of rewards on team capabilities. Finally, empirical analysis on various domains using the framework of multi-agent reinforcement learning highlights important desiderata for multi-agent algorithms towards ensuring CG.

\end{abstract}

\section{Introduction}
\label{sec:intro}
Imagine attending a football summer camp. The coach decides to split the participating players into random teams for practice. While each player has different capabilities (e.g., defending, dribbling, speed, and pace), they quickly adapt to the other players in the team to facilitate the common objective of outscoring their opponents. Furthermore, they smoothly adjust to unexpected events such as a player getting hurt and retiring with substitution, which forces them to change their behaviours and adjust their roles. 
Similarly, they rapidly adjust to changes in team size (as a result of a player being sent off or new players joining the team).
 
Such adaptations are typically possible for two reasons. First, the players understand each others' \textit{capabilities}, including how a change in capabilities affects the underlying environment and chances of success.
Second, players have coordination protocols for adapting to the changes, both explicitly (e.g., communicating the game plan) or implicitly (inferring capabilities from observations, e.g., passing the ball to a player going in for an attack). 
This phenomenon which we term as \textit{Combinatorial Generalization} (CG) is not specific to football or humans, and organisms in general manifest abilities to adapt in almost every situation requiring team efforts \citep{crozier2010masterpiece, nouyan2009teamwork, anderson2003ants}. \am{trim the above 2 paras as needed} \ms{trimmed a little bit.}

Towards capturing specific aspects of CG, recent methods in multi-agent reinforcement learning (MARL) utilize advances in deep learning architectures, such as graph neural networks \cite{ryu2020multi} and attention mechanism \cite{iqbal2021randomized}, as well as extensively tuned training regimes, such as a mixture of human and generated data, self-play, and population-based training \citep{vinyalsGrandmasterLevelStarCraft2019, openaiDotaLargeScale2019}.
While these methods show impressive empirical performance on complex domains, they provide little insight into aspects of when and how much generalization to expect, 
which is crucial for deploying agents in the real world due to
practical considerations like tolerance, minimum expected performance in unseen settings etc. for various deployment scenarios.
Additionally, while the problem of sample-efficient generalization is hard for single-agent RL \citep{mahajan2017asymmetry,duGoodRepresentationSufficient2020,ghoshWhyGeneralizationRL2021, malikWhenGeneralizableReinforcement2021}, it is particularly exacerbated for the multi-agent case.
Specifically, even when the underlying task remains the same, agents in MARL typically need to be trained from scratch for different team compositions. Moreover, across similar tasks with similar team compositions, there is a lack of modularity for sharing knowledge to enable quick learning \citep{wang2020rode}. Thus, we posit that a theoretical understanding of generalization in multi-agent systems (MAS) can help address both of the above-mentioned issues: it can provide important performance guarantees for practical deployment and can additionally inform better algorithm design to ensure sample efficiency.

We first highlight the key properties that make CG particularly difficult for MAS:
\begin{itemize}
	\item \textbf{P1:} The capabilities of agents can come from infinite sets, e.g., maximum permissible torque for an agent joint which can take values in a continuous set.\tr{not specific to MARL though}
	\item \textbf{P2:} Combinatorial blow-up in the number of possible teams (w.r.t.\ agent capabilities) given a team size.
	\item \textbf{P3:} The capabilities need to be grounded w.r.t.\ the dynamics of the environment which becomes increasingly hard with team size (similar to credit assignment).
	\item \textbf{P4:} Team sizes can vary across different tasks. \tr{merge with P2?}
	\item \textbf{P5:} Agents need to infer the capabilities of teammates in settings where it is hidden, in a potentially non-stationary environment. \tr{merge with P1?}
\end{itemize}  
\textbf{P2}-\textbf{P4} particularly distinguish CG from single-agent generalization, highlighting its combinatorial nature. Furthermore, \textbf{P5} requires agents to adapt to changing teammate policies making the problem harder. 

In this work, we analyse multi-agent generalization by modelling the dependence of underlying environment rewards and transitions on agent capabilities. We first look at generalization bounds for the case when the environment dynamics are linear with respect to the agent capabilities. We elucidate how this generalizes the successor feature (SF) framework \cite{barreto2016successor} to the multi-agent case. We provide theoretical bounds for generalization between team compositions, transfer of optimal policy from one team to another and changes to optimal values arising from agent addition and elimination under this framework. Next, we bound the performance gap as a result of an error in estimating the agent capabilities which covers scenarios such as lossy or inaccurate communication. Further, we provide bounds for optimal value deviation when the dynamics themselves are approximately linear. Finally, we elucidate how the bounds can be extended to Lipschitz rewards (\cref{app:lip_rews}) and then extend this framework to study arbitrary dependence of rewards on capabilities to shed light on when generalization can be difficult (\cref{app:gen_rews}). Our results apply to various training and deployment settings in MAS and are agnostic to the type of algorithm used (MARL or other forms of policy search methods). Finally, we empirically analyse popular methods in MARL on tasks designed to offer varying difficulty in terms of generalization and discuss important desiderata to be met for better generalization. 

\section{Background and Formulation}

\subsection*{Multi-Agent Reinforcement Learning}
We model the cooperative multi-agent task as a decentralized partially observable MDP (Dec-POMDP)~\cite{oliehoek_concise_2016}. \am{perhaps give some results under the linear setting?} A Dec-POMDP is formally defined as a tuple $G=\left\langle S,U,P,R,Z,O,n,\rho,\gamma\right\rangle$.
$S$ is the state space of the environment, $\rho$ is the initial state distribution. At each time step $t$, every agent $ i \in \mathcal{A} \equiv \{1,...,n\}$ chooses an action $u^i \in U$ which forms the joint action $\mathbf{u}\in\mathbf{U}\equiv U^n$.
$P(s'|s,\mathbf{u}):S\times\mathbf{U}\times S\rightarrow [0,1]$ is the state transition function. 
$R(s):S \rightarrow [0,1]$ 
is the reward function shared by all agents and $\gamma \in [0,1)$ is the discount factor. 
A Dec-POMDP is \textit{partially observable} \citep{kaelbling1998planning}: each agent $i$ does not have access to the full state and instead samples observations $z\in Z$ according to observation distribution $O(s, i):S \times \mathcal{A}\rightarrow \mathcal{P}(Z)$. Without loss of generality (WLOG), we assume the state is a represented as a $k$-dimensional feature vector $S \subset [0, 1]^k$ and similarly observations $Z \subset [0, 1]^l$. When the observation function $O$ is identity, the problem becomes a multi-agent MDP (MMDP). Similarly, when the observations are invertible for each agent, so that the observation space is partitioned w.r.t.\ $S$, i.e., $\forall i \in \mathcal{A},\forall s_1,s_2 \in S, \forall {z_i} \in Z, P({z_i}|s_1)>0 \land s_1\neq s_2 \implies P({z_i}|s_2)=0$, we classify the problem as a multi-agent richly observed MDP (M-ROMDP) \cite{pmlr-v139-mahajan21a}. 
The action-observation history for an agent $i$ is $\tau^i\in T\equiv(Z\times U)^*$.
We use $u^{-i}$ to denote the action of all the agents other than $i$ and similarly for the policies $\pi^{-i}$. 
The value of a policy is defined as $ V^\pi =  \mathbb{E}_{\pi,\rho} \left[  \sum_{t=0}^\infty \gamma^t R_{\mathcal{T}}(s_t) \right]$. Similarly, the joint action-value function given a policy $\pi$ is defined as: $Q^\pi(s_t, \mathbf{u}_t)=\mathbb{E}_{\pi} \left[\sum^{\infty}_{k=0}\gamma^kR(s_{t+k})|s_t,\mathbf{u}_t\right]$.
The goal is to find the optimal policy $\pi^{*}$ corresponding to the optimal value function $V^*$.

\subsection*{MARL with Agent Capabilities}
We now extend the MARL problem setting for generalisation where agents can have different capabilities.
To this end, we assume that each agent in the task can be characterised by a $d$-dimensional  \emph{capability vector} $c\in\mathcal{C}$, which governs its contribution to rewards and transition dynamics (and thus its policy/behaviour denoted as $\pi^i(~.~;c)$). Without loss of generality, we assume $\mathcal{C} \subseteq \Delta_{d-1}$ (the $d-1$ dimensional simplex). 
Intuitively, an agent's capability reflects the abilities of an agent along various properties that may be important for solving the collective task (e.g., an agent's speed, health recovery, and accuracy).\tr{do you have qualitative results that support such semantically distinct capabilities are represented via c?} 
We next assume an unknown probability distribution $\mathcal{M}:\mathcal{C}^n\rightarrow \mathbb{R}^+$ with support $Sup(\mathcal{M})$ over a subset of the joint capability space $\mathcal{C}^n$.
Any $\mathcal{T}$ sampled from $\mathcal{M}$ can be seen as a tuple of capability vectors $\mathcal{T}=(c_i)_{i=1}^n$, one for each agent in the team. 
We augment the Dec-POMDP with $\mathcal{T}$: $G=\langle S,U,P_{\mathcal{T}},R_{\mathcal{T}},Z,O,n,\rho,\gamma, \mathcal{T}\rangle$ and call it a \textit{variation} for the MARL setting \footnote{Agent capabilities can also be interpreted as the contexts, see \cite{hallak2015contextual}}. Thus $\mathcal{T}$ defines the rewards and transition dynamics of the underlying MMDP (ie. $R_{\mathcal{T}}(s) = \langle f(\mathcal{T}) \cdot s \rangle$ where $\langle \cdot \rangle$ is the dot product\footnote{Note that this is still the most general form as states can be encoded as one-hot vectors, see \cite{barreto2016successor}.} and $f:\mathcal{C}^n \to \mathbb{R}^k$ and similarly for transitions).
Our goal is then to find algorithms, which when trained on a small number of \textit{variations} sampled from $\mathcal{M}$ : $\{\mathcal{T}^j\}_{j=1}^M$, generalise well to unseen team variations in $\mathcal{M}$.
i.e., we want to maximise the expected value over the team variation distribution,
\begin{equation}
\max_{\pi} ~ \mathbb{E}_{\mathcal{T}\sim\mathcal{M}} \left[ \mathbb{E}_{\pi(\cdot; \mathcal{T}),P_\mathcal{T},\rho} \left[  \sum_{t=0}^\infty \gamma^t R_{\mathcal{T}}(s_t) \right] \right],
\label{eqgenobj}
\end{equation}
where $\pi=\{\pi^i\}_{i=1}^n$ is a group of $n$ agents.
The challenge here arises because of two main factors. First, the agents do not have any prior knowledge about what these capability vectors mean, and are thus required to learn their semantics (also called grounding). Second, in the setting where the agents cannot observe the capability vectors (including possibly their own), they have to infer and learn protocols for sharing them with each other in order to generalize in a zero-shot setting. 

\ms{Formally introduce CG!}

\subsection*{Successor Features}
\label{sec:sf}

\newcommand{\mat}[1]{\ensuremath{\boldsymbol{\mathrm{#1}}}}
\newcommand{\vpi}{\ensuremath{\mat{v}^{\pi}}}
\newcommand{\rpi}{\ensuremath{\mat{r}^{\pi}}}
\newcommand{\Ppi}{\ensuremath{\mat{P}^{\pi}}}
\newcommand{\vphi}{\mat{\phi}}
\newcommand{\vpsi}{\mat{\psi}}
\newcommand{\vprm}{\mat{\z}}
\newcommand{\w}{\mat{w}}
\renewcommand{\t}{\ensuremath{\top}}
\renewcommand{\S}{\ensuremath{\mathcal{S}}}
\newcommand{\A}{\ensuremath{\mathcal{A}}}
\newcommand{\W}{\ensuremath{\mathcal{W}}}
\newcommand{\T}{\ensuremath{\mathcal{T}}}
\newcommand{\R}{\ensuremath{\mathbb{R}}}
\newcommand{\M}{\ensuremath{\mathcal{M}^{\phi}}}
\newcommand{\MM}{\ensuremath{\mathcal{M}}} 
\newcommand{\E}{\ensuremath{\mathrm{E}}}

SF framework assumes that the rewards in an MDP can be decomposed as
$r(s) = \vphi(s)^\t\w$, 
where $\vphi(s) \in \R^{d}$ are features of $s$ and 
$\w \in \R^{d}$ are weights \footnote{Similar formulations hold WLOG for \vphi(s,a),\vphi(s,a,s')}. 
When no assumptions is made about $\vphi(s)$, any reward function can be recovered using this representation.
The value function then follows
 \begin{align}
  \label{eqsf} 
  \nonumber V^{\pi}(s) 
  & = \E^{\pi}\left[ r_{t+1} + \gamma r_{t+2} +  ... \,|\, S_t = s \right]  \\
\nonumber & = \E^{\pi}\left[ \vphi_{t+1}^{\t} \w + \gamma \vphi_{t+2}^{\t} \w 
+ ... \, | \, S_t = s \right] \\
  & = \vpsi^{\pi}(s)^{\t} \w.
 \end{align}
Here $\vpsi^{\pi}(s)$ is called the \emph{successor feature} of $s$ under policy $\pi$ \cite{dayan1993improving,barreto2016successor, barreto2018transfer, barreto2020fast}. 
The $i$th component of SF $\vpsi^\pi(s)$ provides the 
expected discounted sum of $\phi_i$ when following policy $\pi$ from 
$s$. \am{maybe change to state dependence only as we do}
	
\section{Analysis}
\label{sec:analysis}
Our analysis focuses on the generalisation properties w.r.t.\ $\mathcal{M}$. We focus on the case of MMDPs for ease of exposition, but similar results for the more general cases can be obtained by suitable assumptions for identifiability of the state (e.g., M-ROMDP in \cite{pmlr-v139-mahajan21a}). Our results are applicable irrespective of whether agents can observe the capabilities. They are also agnostic to the training and deployment regimes (e.g., centralized or decentralized) and the algorithm being used to find the policy. \textbf{All the proofs can be found in \cref{app:proofs}.}

For the analysis we assume that the rewards and transitions depend linearly on the agents capabilities $c_i$ \am{maybe some more lines on this as this is very general and not to be confused with linear dynamics w.r.t. actions., also cite SF works}\tr{yes, would be good to provide a better intuition how general/restrictive this is}:
\begin{align}
    R_{\mathcal{T}}(s) &= \sum_{i=1}^n a_i \langle c_i \cdot W_Rs\rangle \label{eqlinr}\\
    P_{\mathcal{T}}(s'|s,\mathbf{u}) &= \sum_{i=1}^n a_i \langle c_i \cdot  W_P(s',s,\mathbf{u}) \rangle \label{eqlinp}
\end{align}
where $W_R \in \mathbb{R}^{dk}$ is the reward kernel of the MMDP and defines the dependence of the rewards on each capability component. Similarly in \cref{eqlinp}, $W_P:S\times\mathbf{U}\times S\times \{1..d\} \to [0,1]$ defines the transition kernel of the MMDP so that $P_j(\cdot|s,\mathbf{u}) \triangleq W_P(s,\mathbf{u},j) \in \Delta_{|S|-1} , j \in\{1..d\}$ give the next state distribution as directed by the $j^{th}$ component of the capability and agent $i$'s propensity (unweighted) to make the state transition to $s'$ is given by $\Big\langle c_i \cdot  \Big[ P_1(s'|s,\mathbf{u}) \hdots P_d(s'|s,\mathbf{u}) \Big] \Big\rangle = \langle c_i \cdot  W_P(s',s,\mathbf{u}) \rangle$. Finally $(a_i)_{i=1}^n \in \Delta_{n-1}$ are the \textit{influence weights} of agents which quantify the influence of agent $i$ in determining the rewards and transitions. Under the linear setting, given a policy $\pi$ and capabilities $\mathcal{T}$ we have that value function satisfies $V^{\pi}_{\mathcal{T}} = \sum_{i=1}^n a_i \langle c_i \cdot W_R\mu_{\mathcal{T}}^{\pi}\rangle$
where $\mu_{\mathcal{T}}^{\pi} = E_{\rho, P_\mathcal{T}, \pi}[\gamma^t s_t]$ are the expected discounted state features and similarly for a given state $s$, $V^{\pi}_{\mathcal{T}}(s) = \sum_{i=1}^n a_i \langle c_i^TW_R \cdot \mu_{\mathcal{T}}^{\pi}(s)\rangle$ where $\mu_{\mathcal{T}}^{\pi} = E_{P_\mathcal{T}, \pi}[\gamma^t s_t |s_0=s]$. 
The linear formulation for dynamics generalizes the successor feature \cite{barreto2016successor} formulation to the MAS setting,
this can be seen by noting that when the dependence of transition dynamics on capabilities is dropped (\cref{{eqlinp}}) and only single agent is considered (by considering a one-hot $a$), we get the successor feature formulation with capability of the non zero $a_i$ interpreted as the task weight in \citep{barreto2016successor}(see \cref{sec:sf}). We now present the first result concerning the difference between the optimal values of two different team compositions: 
\begin{theorem}[Generalisation between team compositions]
Let team compositions $\mathcal{T}^x,\mathcal{T}^y \in \mathcal{C}^n$ with influence weights $a^x, a^y \in \Delta_{n-1}$, $s_{max} = \max_{s} ||W_R s||_1$,
\am{this can be bound in terms of $W_R$ norm} $V_{mid} = \frac{1}{2} \max_{s} V^{*}_{\mathcal{T}^y}(s)$, Then\footnote{for $\gamma \in (0, \frac{\sqrt{5}-1}{2})$ we can replace $\frac{1}{\gamma(1-\gamma)}$ by $\frac{1+\gamma}{1-\gamma}$}:
$$|V^{*}_{\mathcal{T}^x}- V^{*}_{\mathcal{T}^y}| \leq \frac{s_{max}+\gamma d V_{mid}}{\gamma(1-\gamma)} \Psi \text{, where}$$
\begin{equation}
\Psi = \Big[|\sum_i a^x_i(\mathcal{T}^x_i -\mathcal{T}^y_i)|_{\infty} + |\sum_i (a^x_i-a^y_i)\mathcal{T}^y_i |_{\infty}\Big]
\label{eqpsidef}
\end{equation}
\label{th:gen}
\end{theorem}
\cref{th:gen} gives an interesting decomposition of an upper bound to the difference of the optimal values between the two team compositions. The first terms in the square brackets on the RHS denotes contributions arising purely from substituting the old capacities with the new one. The second term denotes the contribution arising from a change in how much influence the agents have over the dynamics of the MMDP.

\begin{corollary}[Change in optimal value as a result of agent substitution] Let $\mathcal{T} \in \mathcal{C}^n$ be a team composition with influence weights $a \in \Delta_{n-1}$. If agent $i$ is substituted with $i'$ keeping $a_i$ unchanged such that $|\mathcal{T}_{i'}-\mathcal{T}_{i}|_\infty\leq \epsilon_C$ then the new team ($\mathcal{T}'$) optimal value follows: 
\begin{align}
    |V^{*}_{\mathcal{T}'}- V^{*}_{\mathcal{T}}|\leq \frac{(s_{max}+\gamma d V_{mid})a_i\epsilon_C}{\gamma(1-\gamma)}
\end{align} 
\end{corollary}

We define an important policy concept which captures the absolute optimality for an oracle with access to the capabilities. For the ease of exposition we consider fixed influence weights $a$ and define a metric on the joint capability space as $d_{a}(\mathcal{T}^x, \mathcal{T}^y) = |\sum_i a_i(\mathcal{T}^x_i -\mathcal{T}^y_i)|_{\infty}$. We similarly generalize this metric to distances between sets by taking the infimum of the distances between pairs of points in the cross product $d_a(\mathcal{M}_x,\mathcal{M}_y) \triangleq \inf_{\mathcal{T}^x\in \mathcal{M}_x, \mathcal{T}^y\in \mathcal{M}_y} d_{a}(\mathcal{T}^x, \mathcal{T}^y)$.  

\begin{definition}[Absolute Oracle]
Let $\pi^*_{\mathcal{M}}$ be the oracle policy which optimizes \cref{eqgenobj} ie. $\pi^*_{\mathcal{M}}$ is the multiplexer policy which given a team composition $\mathcal{T}$ behaves identically to the optimal policy for $\mathcal{T}^j$ where $\mathcal{T}^j \in \argmin_{\mathcal{T}^l \in Sup(\mathcal{M})} d_{a}(\mathcal{T}^l, \mathcal{T})$.
\end{definition}

We now answer the question of what happens when agents are trained on specific capabilities but the learnt policy is used on potentially unseen capabilities (this could occur e.g. due to changes in hardware components). 
\begin{theorem}[Transfer of optimal policy] Let $\mathcal{T}^x,\mathcal{T}^y \in \mathcal{C}^n$, $a^x, a^y \in \Delta_{n-1}$, $s_{max} = \max_{s} ||W_R s||_1$, $V_{mid} = \frac{1}{2}\max_{ s} V^{*}_{\mathcal{T}^y}(s)$. Let $\pi_y^*$ be the optimal policy for the team composed of agents with capabilities $\mathcal{T}^y$ and influence weights $a^y$. Then:
$$V^{*}_{\mathcal{T}^x} - V^{\pi_{y}^*}_{\mathcal{T}^x} \leq 2\frac{s_{max}+\gamma d V_{mid}}{\gamma(1-\gamma)} \Psi,$$
where $\Psi$ is defined as in \cref{eqpsidef}.
\label{th:trans}
\end{theorem}
\begin{corollary}[Out of distribution performance]
Let $\mathcal{T}\notin Sup(\mathcal{M})$ be an out of distribution task, we then have that the performance of the absolute oracle policy on $\mathcal{T}$ satisfies:
$$V^{*}_{\mathcal{T}} - V^{\pi^*_{\mathcal{M}}}_{\mathcal{T}} \leq 2\frac{s_{max}+\gamma d V_{mid}}{\gamma(1-\gamma)}d_{a}(\mathcal{T}, Sup(\mathcal{M})),$$
\end{corollary}

We now address the scenarios when the team population changes.
\begin{theorem}[Population decrease bound] For the team composition $\mathcal{T} \in \mathcal{C}^n$ with influence weights $a \in \Delta_{n-1}$. If agent $n$ is eliminated  followed by a re-normalization of influence weights, we have that for the remaining team ($\mathcal{T}^- \triangleq (\mathcal{T})_{i=1}^{n-1}$):
\begin{align}
    |V^{*}_{\mathcal{T}^-}- V^{*}_{\mathcal{T}}|\leq \frac{a_n(s_{max}+\gamma d V_{mid})}{\gamma(1-\gamma)}\Big|\sum_{i=1}^{n-1} \frac{a_i\mathcal{T}_i }{1-a_n} - \mathcal{T}_n \Big|_\infty
\end{align} 
\label{th:elim}
\end{theorem}

The special case when $\sum_{i=1}^{n-1} \frac{a_i\mathcal{T}_i }{1-a_n} = \mathcal{T}_n$ for the linear dynamics formulation when an agent-$n$ can in principle be rendered redundant if the rest of the agents in the team can effectively provide a perfect substitute. In fact, this holds true as long as capacity $\mathcal{T}_n$ can be formed from a convex combination of the capabilities $\mathcal{T}_i, i \in \{1..n-1\}$. The latter case however requires using the corresponding convex coefficients instead of re-normalization. A similar bound can be easily constructed for reusing the policy after an agent eliminated to give the corresponding transfer bound along the lines of \cref{th:trans}.

\begin{corollary}[Population increase bound] For the team composition $\mathcal{T} \in \mathcal{C}^n$ with influence weights $a \in \Delta_{n-1}$. If agent $n+1$ is added with capability $\mathcal{T}_{n+1}$ and weight $a_{n+1}$ (other weights scaled down by $\lambda = 1-a_{n+1}$) we have that for the new team ($\mathcal{T}^+ \triangleq (\mathcal{T}_1.. \mathcal{T}_n, \mathcal{T}_{n+1})$):
\begin{align}
    |V^{*}_{\mathcal{T}^+}- V^{*}_{\mathcal{T}}|\leq \frac{a_{n+1}(s_{max}+\gamma d V_{mid})}{\gamma(1-\gamma)}\Big|\sum_{i=1}^{n} a_i\mathcal{T}_i  - \mathcal{T}_{n+1} \Big|_\infty
\end{align}
\end{corollary}

We next extend the generalization bound \cref{th:gen} to include the scenario where the reward and the transition dynamics are not exactly linear but are approximately linear with deviation $\hat\epsilon_R$,$\hat\epsilon_P$ respectively.

\begin{theorem}[Approximate $\hat\epsilon_R$,$\hat\epsilon_P$ dynamics]
Let $\mathcal{T}^x,\mathcal{T}^y \in \mathcal{C}^n$, $a^x, a^y \in \Delta_{n-1}$ and the dynamics be only approximately linear so that $|R_{\mathcal{T}}(s)- \sum_{i=1}^n a_i \langle c_i \cdot W_Rs\rangle| \leq \hat\epsilon_R$ and $|P_{\mathcal{T}}(s'|s,\mathbf{u}) - \sum_{i=1}^n a_i \langle c_i \cdot  W_P(s',s,\mathbf{u}) \rangle|\leq \hat\epsilon_P$.
Then:
$$|V^{*}_{\mathcal{T}^x}- V^{*}_{\mathcal{T}^y}|\leq \frac{s_{max}+\gamma d V_{mid}}{\gamma(1-\gamma)} \Psi + \frac{2(\hat\epsilon_R+\gamma \hat\epsilon_P V_{mid})}{\gamma(1-\gamma)},$$
where $\Psi$ is defined as in \cref{eqpsidef}.
\label{th:gen_err}
\end{theorem}
The other bounds for transfer and population change can similarly be obtained for the approximate dynamics case.

We now consider the scenario when the capabilities are not directly observed but inferred using an approximator which in-turn introduces some errors in their estimation (this could happen due to noise in observations, inaccurate implicit or explicit communication protocols, etc.).
\begin{theorem}[Error from estimation of capabilities]
For the team composition $\mathcal{T} \in \mathcal{C}^n$ with influence weights $a \in \Delta_{n-1}$. If the agent capabilities are inaccurately inferred as $\hat{\mathcal{T}}$ with $\max_i|\mathcal{T}_i- \hat{\mathcal{T}}_i|_\infty \leq \epsilon_{\mathcal{T}}$ and agents learn the inexact policy $\hat{\pi}^*$ then:
$$|V^{*}_{\mathcal{T}}- V^{\hat{\pi}^*}_{\mathcal{T}}| \leq \frac{2\epsilon_{\mathcal{T}}(s_{max}+\gamma d V_{mid})}{\gamma(1-\gamma)}$$ where $V_{mid} = \frac{1}{2} \max_s V_{\hat{\mathcal{T}}}^*(s)$
\label{th:err_est}
\end{theorem}

We note that all our results can be easily extended to the setting where rewards $R_{\mathcal{T}}(s) = \langle f(\mathcal{T}) \cdot W_Rs \rangle$, $f(\mathcal{T})$ is not linear in capabilities as in \cref{eqlinr} but is Lipschitz with coefficient $L_i$ for $i \in \mathcal{A}$. For eg. \cref{th:gen} becomes: 
\begin{theorem}
For rewards $L_i$ Lipschitz in the capabilities with respect to $|\cdot|_\infty$ norm, the difference in optimal values between team compositions $\mathcal{T}^x, \mathcal{T}^y$ satisfy:
$$|V^{*}_{\mathcal{T}^x}- V^{*}_{\mathcal{T}^y}|\leq \frac{s_{max}\sum_{i=1}^{n} L_i| \mathcal{T}^x_i - \mathcal{T}^y_i|_\infty}{\gamma(1-\gamma)}$$
\label{th:gen_lip}
\end{theorem}
See \cref{app:lip_rews} for proof, which also provides a method for extending the other results in a similar fashion.
For the case of general dependence of $f$ on $\mathcal{T}$ (as is common for dense capability embeddings), see \cref{app:gen_rews}. We also present an insight as to why generalization becomes harder in this setting. 
We provide experiments elucidating the bounds stated above in \cref{sec:gen_bound_exps}.

\section{Experimental Setup}\label{sec:experiments}

We evaluate the ability of existing MARL algorithms to generalize to novel settings where the capabilities of teammates change during the training.
We are interested in evaluating the gap between settings encountered during training and held-out agent configurations reserved for testing. 
Furthermore, we aim to study how well algorithms ground privileged information about teammate capabilities and use that during unseen settings at test time.
Lastly, we evaluate the bounds derived in \cref{sec:analysis} on a simple\tr{what's simple about and why do you need to highlight it?} multi-agent problem.

\subsection{Environments}
\subsubsection{Fruit Forage}
We use the fruit forage task on a grid world to empirically demonstrate the generalisation bounds in \cref{sec:analysis}. On a $8\times 8$ grid world we have $n$ agents and $d$ types of fruit trees. For each agent $i$, $\mathcal{T}_i(j), j \in \{1..d\}$ represents the utility of fruit $j$ for agent $i$. The state vector is appended with the $d$ dimensional binary vector representing whether each of the tree types has foraged at a given time step. The details for the team compositions can be found in \cref{app:fruit}. 
\subsubsection{Predator Prey}
We consider the grid-world version of the multi-agent Predator Prey task where 4 agents have to hunt 4 prey in an $8 \times 8$ grid. Here, both predators and preys have certain capabilities. Specifically, each predator has a parameter describing the hit point damage it can cause the prey. Similarly, the prey comes with variations in health. For example, a prey with a capability of 5 can only be caught if the total capability of agents taking the capture action simultaneously on it have capabilities $\geq$ 5 (such as [1,1,1,2]), otherwise, the whole team receives a penalty $p$. 
Here, we test for generalization to novel team composition where test tasks contain a team composition which has not been encountered during training (PP Unseen Team in Figure \ref{fig:pp_v1_v2}), and additionally test tasks where novel team compositions can also have agent types with capabilities not encountered during training (PP Unseen Team, Agent in Figure \ref{fig:pp_v1_v2}). More details are provided in the \cref{app:pred_prey}.%
\subsubsection{StarCraft II}
To assess the generalization capabilities of modern MARL approaches, we make use of a modified version of StarCraft II unit micromanagement tasks of the SMAC benchmark \cite{samvelyan2019starcraft}. Particularly, we consider novel scenarios featuring three unit types from each race of the game where the team composition changes during training and testing, unlike standard SMAC which is static. %
The opponent's team is always identical to the ally team which ensures that the optimal win rate is close to $1$. In the simple cases (\texttt{10\_Protoss}, \texttt{10\_Zerg}, and \texttt{10\_Terran}), agents are trained on various team formations of 10 units that feature all combinations of one, two, and all three unit types, and is later tested on held out team formations.
In the hard cases (\texttt{10\_Protoss\_Hard}, \texttt{10\_Zerg\_Hard}, and \texttt{10\_Terran\_Hard}), agents are exposed to various team formations including two unit types during training. During testing, however, the agents encounter held-out scenarios featuring scenarios with using all three unit types (see \cref{app:smac} for more details). 
In these tasks, agent capabilities are described as a one-hot encoding of agent types.

To test performance on continuously varying capabilities, we also use variants of the environment where either the health or attack accuracy of certain units are reduced. We randomize these configurations for the allied units during training and later test on held-out team configurations.
We evaluate baselines on the \texttt{3m}, \texttt{2s3z}, \texttt{8m} scenarios from the original benchmark with these modifications. The varying team size also helps understand how grounding the capabilities becomes harder as team size increases. Here agent capabilities are described as their accuracy or health coefficients. Further details are provided in the \cref{app:smac}.

\subsection{Baselines}

Our empirical evaluation is based on various types of MARL algorithms. 
We make use of two popular value-based approaches, QMIX \citep{rashid2020monotonic} and VDN \citep{sunehag2017value} that trained fully decentralized policies in a centralized fashion. We also use policy gradient method PPO \cite{schulmanProximalPolicyOptimization2017} that recently shown good results on various MARL domains, both with decentralised (Independent PPO) \citep{dewitt2020independent} and centralised critics (MAPPO) \cite{yu2021surprising}. We access the performance of all baselines when the information about teammates capabilities are provided as observation (denoted with a `C' in parentheses) and when it is not. \textbf{The evaluation procedure, architectures and training details are presented in \cref{app:architecture}.}

\section{Results and Discussion}\label{sec:results}

\subsection{Generalization Bounds}
\label{sec:gen_bound_exps}
\begin{figure}[htb!]
    \centering
    \subfigure[Theorem 1]{
        \includegraphics[width=0.14\textwidth]{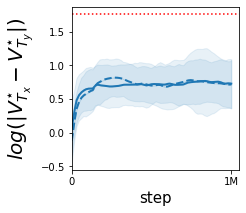}
        \label{fig:th1}
    }
    \subfigure[Theorem 2]{
        \includegraphics[width=0.14\textwidth]{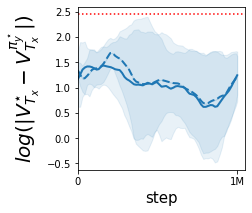}
        \label{fig:th2}
    }
    \subfigure[Theorem 3]{
        \includegraphics[width=0.14\textwidth]{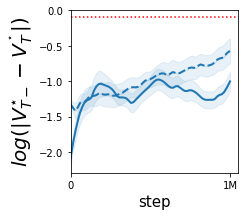}
        \label{fig:th3}
    }
    \caption{Evaluating the bounds for QMIX on Fruit Forage domain. Dashed blue line indicates the setting where agent capabilities are observable. The \textcolor{red}{red dotted line} indicates the corresponding upper bound for each theorem. \am{Bound values T1: 5.81, T2:11.62, T3: 0.91}} %
    \label{fig:theorems}
\end{figure}

\begin{figure*}[htb!]
    \centering
    \includegraphics[width=.43\linewidth]{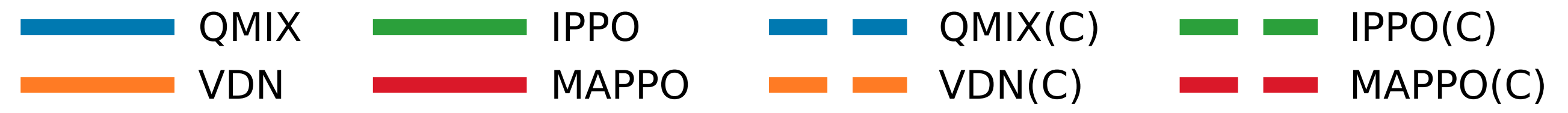}\\
    \subfigure[PP Unseen Team]{
        \includegraphics[width=0.2\linewidth]{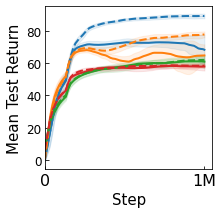}\label{fig:pp1mtr}
    }
    \subfigure[PP Unseen Team]{
        \includegraphics[width=0.2\linewidth]{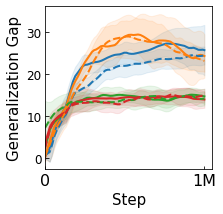}\label{fig:pp1gg}
    }
    \subfigure[PP Unseen Team, Agent]{
        \includegraphics[width=0.2\linewidth]{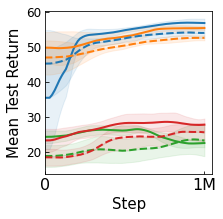}\label{fig:pp2mtr}
    }
    \subfigure[PP Unseen Team, Agent]{
        \includegraphics[width=0.21\linewidth]{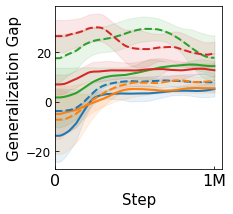}\label{fig:pp2gg}
    }

    \caption{Experimental results for the Predator Prey domain. Standard deviation is shaded.}
    \label{fig:pp_v1_v2}
\end{figure*}

\begin{figure*}[htb!]
    \centering
    \includegraphics[width=.43\linewidth]{figures/plots/sc2_baseline_legend.png}
    \includegraphics[width=\linewidth]{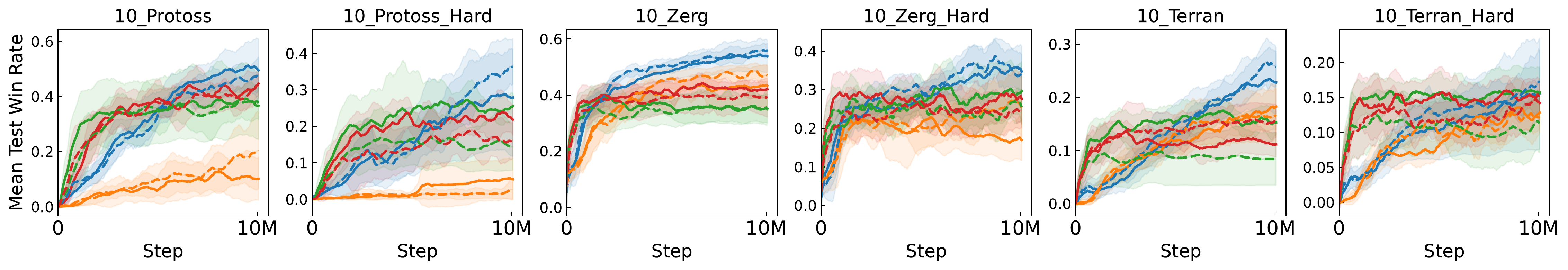}
    \includegraphics[width=\linewidth]{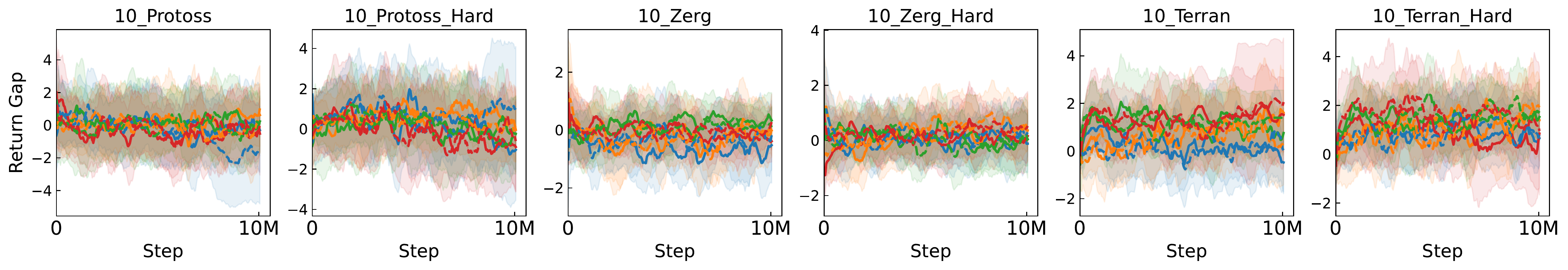}
    \includegraphics[width=\linewidth]{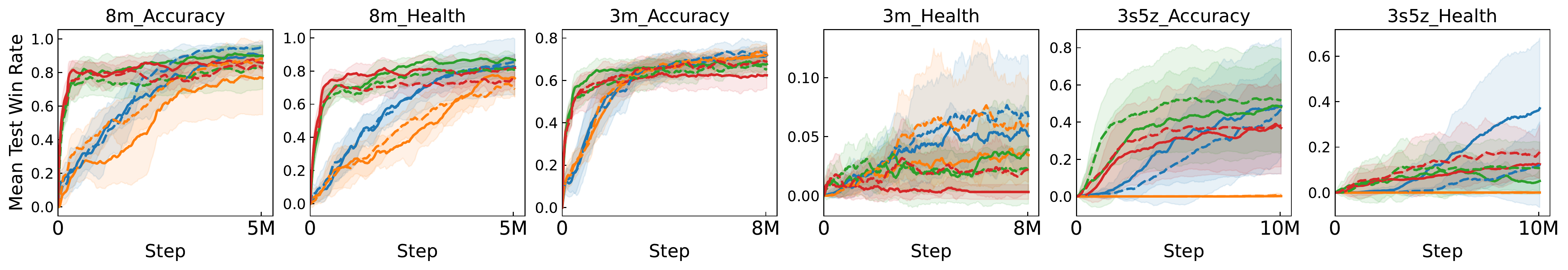}
    \includegraphics[width=\linewidth]{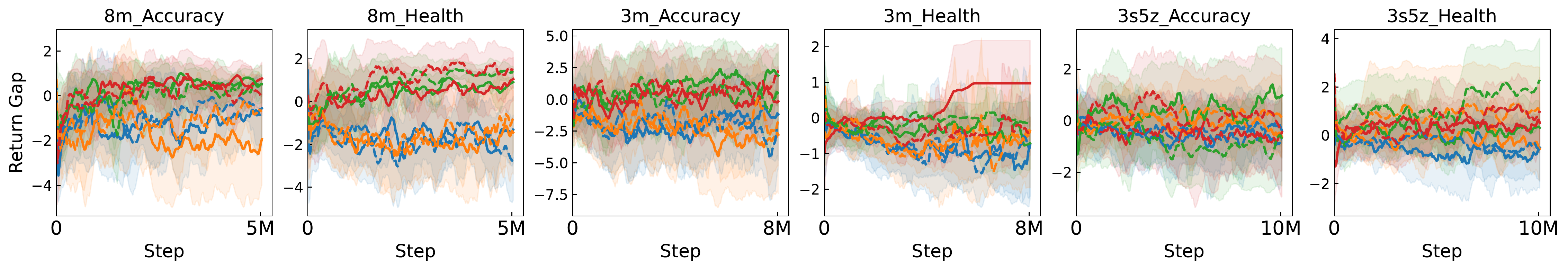}
    \caption{Experimental results on the SMAC benchmark. Standard deviation is shaded. Rows show win rates and generalization gaps.}
    \label{fig:smac}
    \vspace{-3mm}
\end{figure*}

\cref{fig:theorems} provides empirical evaluation of bounds presented in \cref{sec:analysis} in the Fruit Forage domain. We present the plots for training the agents for one million steps of training using QMIX. \cref{fig:th1} shows that the policies in both the domains converge quickly leading to a stable difference in performance thus comfortably satisfying \cref{th:gen}. \cref{fig:th2} showing the gap between optimal and transferred policy shows interesting variations as training proceeds (we posit this happens because the transferred policy becomes steadily specialized thus getting less useful for the target task) the bound in \cref{th:trans} gives a tight fit despite these variations. Finally, we see similarly good fit for the agent elimination scenario in \cref{th:elim} in \cref{fig:th3}.

\subsection{Utilizing Information of Agent Capabilities}

\cref{fig:pp_v1_v2} presents the results of the baselines on Predator Prey domain. We can observe from \cref{fig:pp1mtr} that providing additional information on agent capabilities improves the test-time performance of the baselines with the maximal effect seen on QMIX and VDN. Furthermore, when capabilities are observable to the agents, baselines are able to generalize to new team compositions \cref{fig:pp1gg}, thus successfully grounding the additional information. This hypothesis is additionally supported by the fact that knowing agent capabilities result in a lower generalization gap. Finally, the gap between the settings with known vs. unknown capabilities (dashed vs solid) indicates that agents have likely not come up with any appropriate protocol to communicate their capabilities during test time. %
We also note that the PPO variants do not perform as well as the value-based approaches. Therefore, their low generalization gap \cref{fig:pp1gg} is unlikely representative of good grounding of capability. We posit that this is just because PPO agents are ignoring the privileged information when available. 

For a harder scenario, where both new team composition and agent types appear during evaluation \cref{fig:pp2mtr}, we observe that the situation is reversed from the previous setting as the agents which do not have access to the other's capabilities now perform slightly better. This is strongly indicative of insufficient grounding of the privileged information given to them, which highlights the need for better grounding mechanisms to obtain CG. We see a similar pattern on generalization gap in \cref{fig:pp2gg} where privileged information hurts the performance and is likely perceived as observation noise.

On the more challenging domain of StarCraft, we see that for easier capability variations of health and accuracy (as they are continuous and more readily usable for an agent's immediate actions), knowing the capabilities is advantageous to the agent during test time. Moreover, the relative gains of knowing the privileged information go down as the task difficulty increases. The accuracy variations tend to be easier as typical joint policies like focus fire remain unchanged. Moreover, health variations on the smaller team make the task much harder than bigger team due to relative loss in team hit points. In this regard, \texttt{8m, 3s5z} accuracy versions show good grounding and generalization. This changes as tasks get harder.
On the harder tasks which involve swapping unit types within \texttt{Protoss, Zerg, Terran} races, we observe that knowing the capabilities of other agents gives little advantage. This is especially noticeable on the \texttt{Hard} versions where all unit types are never within a single team during training. Furthermore, with win-rate performances on these maps being low, we hypothesise that the agents do not successfully utilize the capability information. Thus, it is highly unlikely that they learn any meaningful communication protocols for exchanging capability information. 
\textbf{For full StarCraft II results, including \texttt{8m\_vs\_9m} \& \texttt{10m\_vs\_11m} scenarios, see \cref{app:results}}.

Compared to the relatively simple Predator Prey task, generalization in StarCraft proved to be more difficult for the baselines. Although static versions of SMAC environments are comfortably solved by them \cite{rashid2020monotonic, dewitt2020independent, yu2021surprising}, changing unit formations or unit health/accuracy makes the tasks significantly difficult, even for configurations seen during the training. As observed in \cref{fig:smac}, providing the capability information does not consistently improve the test-time performance. This suggests the poor grounding abilities of the baseline algorithms, which reinforces the need for better grounding mechanisms in the MARL algorithms (e.g., forward dynamics prediction as in \cite{jaderberg2016reinforcement}). The failure to generalize on index-based privileged information regarding agent types suggests using mechanisms such as latent embeddings to compose and reason about capabilities. Finally, a low test performance gap between agents having privileged information vs those which do not, coupled with a low generalization gap, suggests that these methods do not facilitate information sharing between the agents, which is another desideratum towards attaining CG.

\section{Related Work}

\textbf{Multi-agent systems} \citep{claus1998dynamics, busoniu_comprehensive_2008} offer means to overcome theoretical barriers like exponential blow up in state-action space and compute resource requirements for large problems. 
MARL is a promising approach for training cooperative MAS. Recent progress in cooperative MARL \citep{lowe_multi-agent_2017, sunehag2017value, rashid2020monotonic, pmlr-v139-mahajan21a} has demonstrated impressive applications in solving complex tasks in games such as StarCraft II \citep{samvelyan2019starcraft}. Specialized methods which improve exploration in MARL have been proposed using hierarchical learning \citep{mahajan2019maven} and successor features \citep{gupta2021uneven}. Methods for factorizing the action space \citep{wang2020rode} have shown improvement in sample complexity. \citet{iqbal2021randomized} regularize value functions to share factors comprised of sub-groups of entities, in order to transfer knowledge across cooperative tasks. In the competitive/general sum MARL space \citep{lowe_multi-agent_2017, openaiDotaLargeScale2019} have shown impressive performance on complex tasks. \citet{pmlr-v119-vezhnevets20a} use an options framework to learn agents which generalize against different opponents. \cite{czarnecki2020real, tuyls2020bounds,piliouras2021evolutionary} explore the structural and theoretical properties of general payoff games. 

\textbf{Ad-hoc coordination} was first formalised by \citet{stone2010ad} by modelling the multi-agent problem as a single-agent task and using competency scores to measure agent compatibility. Methods for using explicit hard-coded protocols for adaptations were explored in \cite{tambe1997towards, grosz1996collaborative}. Opponent modelling for general game was explored in \cite{stone2000defining, markovitch2005learning, ledezma2004predicting}. Several approaches to the ad-hoc cooperation problem assume that the behaviour of other agents in the ensemble are fixed \cite{bowling2005coordination}. 
Planning methods like Monte Carlo tree search are used for finding optimal adaptation policy from a fixed set of choices \citep{barrett2011empirical, albrecht2016belief, albrecht2019reasoning}. \citet{nikolaidis2014efficient} develop over this by enabling learning a set of behaviours for the adapting agent while performing the task with human agents instead of assuming that it is given beforehand. Recent methods allow a change in the behaviour of the other agents to ones picked from a fixed set and account for the possible non-stationarities using change point detection \cite{hernandez2017efficiently, ravula2019ad}. However, these methods do not consider arbitrary learning for other agents. Furthermore, they do not focus on generalization to unseen agent capabilities.

\textbf{Generalization in RL} aims to develop approaches that generalize well to the novel, unseen scenarios after training \citep{kirk2022survey}. Such methods avoid overfitting to seen tasks and can produce robust behaviour when deployed to novel settings. %
Recent work on generalization in single-agent RL make use of techniques such as
data augmentation \citep{raileanu2021automatic, kostrikov2021image},
environment generation \citep{openendedlearningteam2021openended, melting-pot},
encoding inductive biases \citep{higginsDARLAImprovingZeroShot2018}, 
and regularization \citep{pmlr-v97-cobbe19a}.
Methods in contextual MDPs \cite{hallak2015contextual, zhang2020multi} also provide generalization with guarantees. 
Recent work also elucidate some of the fundamental bounds arising from computational complexity which prevents sample efficient generalization \citep{duGoodRepresentationSufficient2020,ghoshWhyGeneralizationRL2021, malikWhenGeneralizableReinforcement2021}.
\vspace{-2mm}	
\section{Conclusion and Future work}
In this work, we studied the generalization properties in multi-agent systems (MAS) following Markovian dynamics with a linear dependence of dynamics on the agent capabilities. We showed how the framework extends the successor feature setting to MAS. We explored performance bounds for various interesting scenarios arising in the MAS including generalization, transfer, agent substitutions, approximate inference of capabilities and deviations in environment dynamics. Furthermore, we showed how the bounds can be extended to the Lipschitz reward setting and elucidated the most general form of rewards and how they make generalization difficult. Finally, we extensively tested the popular MARL algorithms on domains presenting a wide spectrum of hardness for CG. We saw that while some algorithms demonstrated preliminary CG on easier domains, all of the algorithms are insufficient towards ensuring CG on the challenging domains. We further highlighted how the first step towards ensuring CG should be ensuring proper grounding of agent capabilities. For future work, we aim to provide tighter bounds for CG for more general dynamics and create methods for better grounding of agent capabilities.	
\section{Acknowledgements}
Anuj Mahajan is funded by the J.P. Morgan A.I. Fellowship. Tarun Gupta is
supported by the Oxford University Clarendon Scholarship. Part of the experiments were made possible by generous equipment grant by NVIDIA. 	
\clearpage
\newpage
\bibliography{metamarl.bib}

\begin{thebibliography}{58}
\providecommand{\natexlab}[1]{#1}
\providecommand{\url}[1]{\texttt{#1}}
\expandafter\ifx\csname urlstyle\endcsname\relax
  \providecommand{\doi}[1]{doi: #1}\else
  \providecommand{\doi}{doi: \begingroup \urlstyle{rm}\Url}\fi

\bibitem[Albrecht \& Stone(2019)Albrecht and Stone]{albrecht2019reasoning}
Albrecht, S.~V. and Stone, P.
\newblock Reasoning about hypothetical agent behaviours and their parameters.
\newblock \emph{arXiv preprint arXiv:1906.11064}, 2019.

\bibitem[Albrecht et~al.(2016)Albrecht, Crandall, and
  Ramamoorthy]{albrecht2016belief}
Albrecht, S.~V., Crandall, J.~W., and Ramamoorthy, S.
\newblock Belief and truth in hypothesised behaviours.
\newblock \emph{Artificial Intelligence}, 235:\penalty0 63--94, 2016.

\bibitem[Anderson \& McMillan(2003)Anderson and McMillan]{anderson2003ants}
Anderson, C. and McMillan, E.
\newblock Of ants and men: Self-organized teams in human and insect
  organizations.
\newblock \emph{Emergence}, 5\penalty0 (2):\penalty0 29--41, 2003.

\bibitem[Barreto et~al.(2016)Barreto, Dabney, Munos, Hunt, Schaul, Van~Hasselt,
  and Silver]{barreto2016successor}
Barreto, A., Dabney, W., Munos, R., Hunt, J.~J., Schaul, T., Van~Hasselt, H.,
  and Silver, D.
\newblock Successor features for transfer in reinforcement learning.
\newblock \emph{arXiv preprint arXiv:1606.05312}, 2016.

\bibitem[Barreto et~al.(2018)Barreto, Borsa, Quan, Schaul, Silver, Hessel,
  Mankowitz, Zidek, and Munos]{barreto2018transfer}
Barreto, A., Borsa, D., Quan, J., Schaul, T., Silver, D., Hessel, M.,
  Mankowitz, D., Zidek, A., and Munos, R.
\newblock Transfer in deep reinforcement learning using successor features and
  generalised policy improvement.
\newblock In \emph{International Conference on Machine Learning}, pp.\
  501--510. PMLR, 2018.

\bibitem[Barreto et~al.(2020)Barreto, Hou, Borsa, Silver, and
  Precup]{barreto2020fast}
Barreto, A., Hou, S., Borsa, D., Silver, D., and Precup, D.
\newblock Fast reinforcement learning with generalized policy updates.
\newblock \emph{Proceedings of the National Academy of Sciences}, 117\penalty0
  (48):\penalty0 30079--30087, 2020.

\bibitem[Barrett et~al.(2011)Barrett, Stone, and Kraus]{barrett2011empirical}
Barrett, S., Stone, P., and Kraus, S.
\newblock Empirical evaluation of ad hoc teamwork in the pursuit domain.
\newblock In \emph{AAMAS}, pp.\  567--574, 2011.

\bibitem[Bowling \& McCracken(2005)Bowling and
  McCracken]{bowling2005coordination}
Bowling, M. and McCracken, P.
\newblock Coordination and adaptation in impromptu teams.
\newblock In \emph{AAAI}, volume~5, pp.\  53--58, 2005.

\bibitem[Busoniu et~al.(2008)Busoniu, Babuska, and
  De~Schutter]{busoniu_comprehensive_2008}
Busoniu, L., Babuska, R., and De~Schutter, B.
\newblock A {Comprehensive} {Survey} of {Multiagent} {Reinforcement}
  {Learning}.
\newblock \emph{IEEE Transactions on Systems, Man, and Cybernetics, Part C
  (Applications and Reviews)}, 38\penalty0 (2):\penalty0 156--172, 2008.

\bibitem[Claus \& Boutilier(1998)Claus and Boutilier]{claus1998dynamics}
Claus, C. and Boutilier, C.
\newblock The dynamics of reinforcement learning in cooperative multiagent
  systems.
\newblock \emph{AAAI/IAAI}, 1998\penalty0 (746-752):\penalty0 2, 1998.

\bibitem[Cobbe et~al.(2019)Cobbe, Klimov, Hesse, Kim, and
  Schulman]{pmlr-v97-cobbe19a}
Cobbe, K., Klimov, O., Hesse, C., Kim, T., and Schulman, J.
\newblock Quantifying generalization in reinforcement learning.
\newblock In Chaudhuri, K. and Salakhutdinov, R. (eds.), \emph{Proceedings of
  the 36th International Conference on Machine Learning}, volume~97 of
  \emph{Proceedings of Machine Learning Research}, pp.\  1282--1289. PMLR,
  09--15 Jun 2019.
\newblock URL \url{https://proceedings.mlr.press/v97/cobbe19a.html}.

\bibitem[Crozier et~al.(2010)Crozier, Newey, Schluens, Robson,
  et~al.]{crozier2010masterpiece}
Crozier, R.~H., Newey, P.~S., Schluens, E.~A., Robson, S.~K., et~al.
\newblock A masterpiece of evolution--oecophylla weaver ants (hymenoptera:
  Formicidae).
\newblock \emph{Myrmecological News}, 13\penalty0 (5), 2010.

\bibitem[Czarnecki et~al.(2020)Czarnecki, Gidel, Tracey, Tuyls, Omidshafiei,
  Balduzzi, and Jaderberg]{czarnecki2020real}
Czarnecki, W.~M., Gidel, G., Tracey, B., Tuyls, K., Omidshafiei, S., Balduzzi,
  D., and Jaderberg, M.
\newblock Real world games look like spinning tops, 2020.

\bibitem[Dayan(1993)]{dayan1993improving}
Dayan, P.
\newblock Improving generalization for temporal difference learning: The
  successor representation.
\newblock \emph{Neural Computation}, 5\penalty0 (4):\penalty0 613--624, 1993.

\bibitem[de~Witt et~al.(2020)de~Witt, Gupta, Makoviichuk, Makoviychuk, Torr,
  Sun, and Whiteson]{dewitt2020independent}
de~Witt, C.~S., Gupta, T., Makoviichuk, D., Makoviychuk, V., Torr, P.~H., Sun,
  M., and Whiteson, S.
\newblock Is independent learning all you need in the starcraft multi-agent
  challenge?
\newblock \emph{arXiv preprint arXiv:2011.09533}, 2020.

\bibitem[Du et~al.(2020)Du, Kakade, Wang, and
  Yang]{duGoodRepresentationSufficient2020}
Du, S.~S., Kakade, S.~M., Wang, R., and Yang, L.~F.
\newblock Is a good representation sufficient for sample efficient
  reinforcement learning?
\newblock In \emph{8th International Conference on Learning Representations,
  {ICLR} 2020, Addis Ababa, Ethiopia, April 26-30, 2020}. OpenReview.net, 2020.
\newblock URL \url{https://openreview.net/forum?id=r1genAVKPB}.

\bibitem[Ghosh et~al.(2021)Ghosh, Rahme, Kumar, Zhang, Adams, and
  Levine]{ghoshWhyGeneralizationRL2021}
Ghosh, D., Rahme, J., Kumar, A., Zhang, A., Adams, R.~P., and Levine, S.
\newblock Why {{Generalization}} in {{RL}} is {{Difficult}}: {{Epistemic
  POMDPs}} and {{Implicit Partial Observability}}.
\newblock \emph{arXiv:2107.06277 [cs, stat]}, 2021.
\newblock URL \url{http://arxiv.org/abs/2107.06277}.

\bibitem[Grosz \& Kraus(1996)Grosz and Kraus]{grosz1996collaborative}
Grosz, B. and Kraus, S.
\newblock Collaborative plans for complex group action.
\newblock \emph{Artificial Intelligence}, 1996.

\bibitem[Gupta et~al.(2021)Gupta, Mahajan, Peng, B{\"o}hmer, and
  Whiteson]{gupta2021uneven}
Gupta, T., Mahajan, A., Peng, B., B{\"o}hmer, W., and Whiteson, S.
\newblock Uneven: Universal value exploration for multi-agent reinforcement
  learning.
\newblock In \emph{International Conference on Machine Learning}, pp.\
  3930--3941. PMLR, 2021.

\bibitem[Hallak et~al.(2015)Hallak, Castro, and Mannor]{hallak2015contextual}
Hallak, A., Castro, D.~D., and Mannor, S.
\newblock Contextual markov decision processes, 2015.

\bibitem[Hausknecht \& Stone(2015)Hausknecht and Stone]{hausknecht_deep_2015}
Hausknecht, M. and Stone, P.
\newblock Deep {Recurrent} {Q}-{Learning} for {Partially} {Observable} {MDPs}.
\newblock In \emph{AAAI Fall Symposium on Sequential Decision Making for
  Intelligent Agents}, 2015.

\bibitem[Hernandez-Leal et~al.(2017)Hernandez-Leal, Zhan, Taylor, Sucar, and
  De~Cote]{hernandez2017efficiently}
Hernandez-Leal, P., Zhan, Y., Taylor, M.~E., Sucar, L.~E., and De~Cote, E.~M.
\newblock Efficiently detecting switches against non-stationary opponents.
\newblock \emph{Autonomous Agents and Multi-Agent Systems}, 31\penalty0
  (4):\penalty0 767--789, 2017.

\bibitem[Higgins et~al.(2017)Higgins, Pal, Rusu, Matthey, Burgess, Pritzel,
  Botvinick, Blundell, and Lerchner]{higginsDARLAImprovingZeroShot2018}
Higgins, I., Pal, A., Rusu, A.~A., Matthey, L., Burgess, C., Pritzel, A.,
  Botvinick, M., Blundell, C., and Lerchner, A.
\newblock {DARLA:} improving zero-shot transfer in reinforcement learning.
\newblock In Precup, D. and Teh, Y.~W. (eds.), \emph{Proceedings of the 34th
  International Conference on Machine Learning, {ICML} 2017, Sydney, NSW,
  Australia, 6-11 August 2017}, volume~70 of \emph{Proceedings of Machine
  Learning Research}, pp.\  1480--1490. {PMLR}, 2017.
\newblock URL \url{http://proceedings.mlr.press/v70/higgins17a.html}.

\bibitem[Iqbal et~al.(2021)Iqbal, de~Witt, Peng, Böhmer, Whiteson, and
  Sha]{iqbal2021randomized}
Iqbal, S., de~Witt, C. A.~S., Peng, B., Böhmer, W., Whiteson, S., and Sha, F.
\newblock Randomized entity-wise factorization for multi-agent reinforcement
  learning, 2021.

\bibitem[Jaderberg et~al.(2016)Jaderberg, Mnih, Czarnecki, Schaul, Leibo,
  Silver, and Kavukcuoglu]{jaderberg2016reinforcement}
Jaderberg, M., Mnih, V., Czarnecki, W.~M., Schaul, T., Leibo, J.~Z., Silver,
  D., and Kavukcuoglu, K.
\newblock Reinforcement learning with unsupervised auxiliary tasks, 2016.

\bibitem[Kaelbling et~al.(1998)Kaelbling, Littman, and
  Cassandra]{kaelbling1998planning}
Kaelbling, L.~P., Littman, M.~L., and Cassandra, A.~R.
\newblock Planning and acting in partially observable stochastic domains.
\newblock \emph{Artificial intelligence}, 101\penalty0 (1-2):\penalty0 99--134,
  1998.

\bibitem[Kirk et~al.(2022)Kirk, Zhang, Grefenstette, and
  Rocktäschel]{kirk2022survey}
Kirk, R., Zhang, A., Grefenstette, E., and Rocktäschel, T.
\newblock A survey of generalisation in deep reinforcement learning, 2022.

\bibitem[Kostrikov et~al.(2021)Kostrikov, Yarats, and
  Fergus]{kostrikov2021image}
Kostrikov, I., Yarats, D., and Fergus, R.
\newblock Image augmentation is all you need: Regularizing deep reinforcement
  learning from pixels, 2021.

\bibitem[Ledezma et~al.(2004)Ledezma, Aler, Sanchis, and
  Borrajo]{ledezma2004predicting}
Ledezma, A., Aler, R., Sanchis, A., and Borrajo, D.
\newblock Predicting opponent actions by observation.
\newblock In \emph{Robot Soccer World Cup}, pp.\  286--296. Springer, 2004.

\bibitem[Leibo et~al.(2021)Leibo, Due{\~n}ez-Guzman, Vezhnevets, Agapiou,
  Sunehag, Koster, Matyas, Beattie, Mordatch, and Graepel]{melting-pot}
Leibo, J.~Z., Due{\~n}ez-Guzman, E.~A., Vezhnevets, A., Agapiou, J.~P.,
  Sunehag, P., Koster, R., Matyas, J., Beattie, C., Mordatch, I., and Graepel,
  T.
\newblock Scalable evaluation of multi-agent reinforcement learning with
  melting pot.
\newblock In \emph{Proceedings of the 38th International Conference on Machine
  Learning}, pp.\  6187--6199. PMLR, 2021.

\bibitem[Lowe et~al.(2017)Lowe, Wu, Tamar, Harb, Abbeel, and
  Mordatch]{lowe_multi-agent_2017}
Lowe, R., Wu, Y., Tamar, A., Harb, J., Abbeel, O.~P., and Mordatch, I.
\newblock Multi-agent actor-critic for mixed cooperative-competitive
  environments.
\newblock In \emph{Advances in Neural Information Processing Systems}, pp.\
  6382--6393, 2017.

\bibitem[Mahajan \& Tulabandhula(2017)Mahajan and
  Tulabandhula]{mahajan2017asymmetry}
Mahajan, A. and Tulabandhula, T.
\newblock Symmetry learning for function approximation in reinforcement
  learning.
\newblock \emph{arXiv preprint arXiv:1706.02999}, 2017.

\bibitem[Mahajan et~al.(2019)Mahajan, Rashid, Samvelyan, and
  Whiteson]{mahajan2019maven}
Mahajan, A., Rashid, T., Samvelyan, M., and Whiteson, S.
\newblock Maven: Multi-agent variational exploration.
\newblock In \emph{Advances in Neural Information Processing Systems}, pp.\
  7613--7624, 2019.

\bibitem[Mahajan et~al.(2021)Mahajan, Samvelyan, Mao, Makoviychuk, Garg,
  Kossaifi, Whiteson, Zhu, and Anandkumar]{pmlr-v139-mahajan21a}
Mahajan, A., Samvelyan, M., Mao, L., Makoviychuk, V., Garg, A., Kossaifi, J.,
  Whiteson, S., Zhu, Y., and Anandkumar, A.
\newblock Tesseract: Tensorised actors for multi-agent reinforcement learning.
\newblock In \emph{Proceedings of the 38th International Conference on Machine
  Learning}, volume 139, pp.\  7301--7312. PMLR, 2021.
\newblock URL \url{https://proceedings.mlr.press/v139/mahajan21a.html}.

\bibitem[Malik et~al.(2021)Malik, Li, and
  Ravikumar]{malikWhenGeneralizableReinforcement2021}
Malik, D., Li, Y., and Ravikumar, P.
\newblock When {{Is Generalizable Reinforcement Learning Tractable}}?
\newblock \emph{arXiv:2101.00300 [cs, stat]}, 2021.
\newblock URL \url{http://arxiv.org/abs/2101.00300}.

\bibitem[Markovitch \& Reger(2005)Markovitch and Reger]{markovitch2005learning}
Markovitch, S. and Reger, R.
\newblock Learning and exploiting relative weaknesses of opponent agents.
\newblock \emph{Autonomous Agents and Multi-Agent Systems}, 10\penalty0
  (2):\penalty0 103--130, 2005.

\bibitem[Nikolaidis et~al.(2014)Nikolaidis, Gu, Ramakrishnan, and
  Shah]{nikolaidis2014efficient}
Nikolaidis, S., Gu, K., Ramakrishnan, R., and Shah, J.
\newblock Efficient model learning for human-robot collaborative tasks. arxiv,
  2014.

\bibitem[Nouyan et~al.(2009)Nouyan, Gro{\ss}, Bonani, Mondada, and
  Dorigo]{nouyan2009teamwork}
Nouyan, S., Gro{\ss}, R., Bonani, M., Mondada, F., and Dorigo, M.
\newblock Teamwork in self-organized robot colonies.
\newblock \emph{IEEE Transactions on Evolutionary Computation}, 13\penalty0
  (4):\penalty0 695--711, 2009.

\bibitem[Oliehoek \& Amato(2016)Oliehoek and Amato]{oliehoek_concise_2016}
Oliehoek, F.~A. and Amato, C.
\newblock \emph{A {Concise} {Introduction} to {Decentralized} {POMDPs}}.
\newblock {SpringerBriefs} in {Intelligent} {Systems}. Springer, 2016.

\bibitem[OpenAI et~al.(2019)OpenAI, Berner, Brockman, Chan, Cheung,
  D{\k{e}}biak, Dennison, Farhi, Fischer, Hashme, Hesse, J{\'o}zefowicz, Gray,
  Olsson, Pachocki, Petrov, Pinto, Raiman, Salimans, Schlatter, Schneider,
  Sidor, Sutskever, Tang, Wolski, and Zhang]{openaiDotaLargeScale2019}
OpenAI, Berner, C., Brockman, G., Chan, B., Cheung, V., Debiak, P.,
  Dennison, C., Farhi, D., Fischer, Q., Hashme, S., Hesse, C., J{\'o}zefowicz,
  R., Gray, S., Olsson, C., Pachocki, J., Petrov, M., Pinto, H. P. d.~O.,
  Raiman, J., Salimans, T., Schlatter, J., Schneider, J., Sidor, S., Sutskever,
  I., Tang, J., Wolski, F., and Zhang, S.
\newblock Dota 2 with {{Large Scale Deep Reinforcement Learning}}.
\newblock \emph{arXiv:1912.06680 [cs, stat]}, 2019.
\newblock URL \url{http://arxiv.org/abs/1912.06680}.

\bibitem[Piliouras et~al.(2021)Piliouras, Rowland, Omidshafiei, Elie, Hennes,
  Connor, and Tuyls]{piliouras2021evolutionary}
Piliouras, G., Rowland, M., Omidshafiei, S., Elie, R., Hennes, D., Connor, J.,
  and Tuyls, K.
\newblock Evolutionary dynamics and $\phi$-regret minimization in games, 2021.

\bibitem[Raileanu et~al.(2021)Raileanu, Goldstein, Yarats, Kostrikov, and
  Fergus]{raileanu2021automatic}
Raileanu, R., Goldstein, M., Yarats, D., Kostrikov, I., and Fergus, R.
\newblock Automatic data augmentation for generalization in deep reinforcement
  learning, 2021.

\bibitem[Rashid et~al.(2020)Rashid, Samvelyan, de~Witt, Farquhar, Foerster, and
  Whiteson]{rashid2020monotonic}
Rashid, T., Samvelyan, M., de~Witt, C.~S., Farquhar, G., Foerster, J., and
  Whiteson, S.
\newblock Monotonic value function factorisation for deep multi-agent
  reinforcement learning.
\newblock \emph{Journal of Machine Learning Research}, 21\penalty0
  (178):\penalty0 1--51, 2020.
\newblock URL \url{http://jmlr.org/papers/v21/20-081.html}.

\bibitem[Ravula(2019)]{ravula2019ad}
Ravula, M. C.~R.
\newblock \emph{Ad-hoc teamwork with behavior-switching agents}.
\newblock PhD thesis, 2019.

\bibitem[Ryu et~al.(2020)Ryu, Shin, and Park]{ryu2020multi}
Ryu, H., Shin, H., and Park, J.
\newblock Multi-agent actor-critic with hierarchical graph attention network.
\newblock In \emph{Proceedings of the AAAI Conference on Artificial
  Intelligence}, volume~34, pp.\  7236--7243, 2020.

\bibitem[Samvelyan et~al.(2019)Samvelyan, Rashid, de~Witt, Farquhar, Nardelli,
  Rudner, Hung, Torr, Foerster, and Whiteson]{samvelyan2019starcraft}
Samvelyan, M., Rashid, T., de~Witt, C.~S., Farquhar, G., Nardelli, N., Rudner,
  T.~G., Hung, C.-M., Torr, P.~H., Foerster, J., and Whiteson, S.
\newblock {The} {StarCraft} {Multi-Agent} {Challenge}.
\newblock In \emph{Proceedings of the 18th International Conference on
  Autonomous Agents and MultiAgent Systems}, 2019.

\bibitem[Schulman et~al.(2017)Schulman, Wolski, Dhariwal, Radford, and
  Klimov]{schulmanProximalPolicyOptimization2017}
Schulman, J., Wolski, F., Dhariwal, P., Radford, A., and Klimov, O.
\newblock Proximal {{Policy Optimization Algorithms}}.
\newblock \emph{arXiv:1707.06347 [cs]}, 2017.
\newblock URL \url{http://arxiv.org/abs/1707.06347}.

\bibitem[Stone et~al.(2000)Stone, Riley, and Veloso]{stone2000defining}
Stone, P., Riley, P., and Veloso, M.
\newblock Defining and using ideal teammate and opponent agent models: A case
  study in robotic soccer.
\newblock In \emph{Proceedings Fourth International Conference on MultiAgent
  Systems}, pp.\  441--442. IEEE, 2000.

\bibitem[Stone et~al.(2010)Stone, Kaminka, Kraus, and Rosenschein]{stone2010ad}
Stone, P., Kaminka, G.~A., Kraus, S., and Rosenschein, J.~S.
\newblock Ad hoc autonomous agent teams: Collaboration without
  pre-coordination.
\newblock In \emph{Twenty-Fourth AAAI Conference on Artificial Intelligence},
  2010.

\bibitem[Sunehag et~al.(2017)Sunehag, Lever, Gruslys, Czarnecki, Zambaldi,
  Jaderberg, Lanctot, Sonnerat, Leibo, Tuyls, et~al.]{sunehag2017value}
Sunehag, P., Lever, G., Gruslys, A., Czarnecki, W.~M., Zambaldi, V., Jaderberg,
  M., Lanctot, M., Sonnerat, N., Leibo, J.~Z., Tuyls, K., et~al.
\newblock Value-decomposition networks for cooperative multi-agent learning.
\newblock \emph{arXiv preprint arXiv:1706.05296}, 2017.

\bibitem[Tambe(1997)]{tambe1997towards}
Tambe, M.
\newblock Towards flexible teamwork.
\newblock \emph{Journal of artificial intelligence research}, 7:\penalty0
  83--124, 1997.

\bibitem[Team et~al.(2021)Team, Stooke, Mahajan, Barros, Deck, Bauer,
  Sygnowski, Trebacz, Jaderberg, Mathieu, McAleese, Bradley-Schmieg, Wong,
  Porcel, Raileanu, Hughes-Fitt, Dalibard, and
  Czarnecki]{openendedlearningteam2021openended}
Team, O. E.~L., Stooke, A., Mahajan, A., Barros, C., Deck, C., Bauer, J.,
  Sygnowski, J., Trebacz, M., Jaderberg, M., Mathieu, M., McAleese, N.,
  Bradley-Schmieg, N., Wong, N., Porcel, N., Raileanu, R., Hughes-Fitt, S.,
  Dalibard, V., and Czarnecki, W.~M.
\newblock Open-ended learning leads to generally capable agents, 2021.

\bibitem[Tuyls et~al.(2020)Tuyls, Perolat, Lanctot, Hughes, Everett, Leibo,
  Szepesv{\'a}ri, and Graepel]{tuyls2020bounds}
Tuyls, K., Perolat, J., Lanctot, M., Hughes, E., Everett, R., Leibo, J.~Z.,
  Szepesv{\'a}ri, C., and Graepel, T.
\newblock Bounds and dynamics for empirical game theoretic analysis.
\newblock \emph{Autonomous Agents and Multi-Agent Systems}, 34\penalty0
  (1):\penalty0 1--30, 2020.

\bibitem[Vezhnevets et~al.(2020)Vezhnevets, Wu, Eckstein, Leblond, and
  Leibo]{pmlr-v119-vezhnevets20a}
Vezhnevets, A., Wu, Y., Eckstein, M., Leblond, R., and Leibo, J.~Z.
\newblock {OP}tions as {RE}sponses: Grounding behavioural hierarchies in
  multi-agent reinforcement learning.
\newblock In \emph{Proceedings of the 37th International Conference on Machine
  Learning}, Proceedings of Machine Learning Research, pp.\  9733--9742. PMLR,
  2020.

\bibitem[Vinyals et~al.(2019)Vinyals, Babuschkin, Czarnecki, Mathieu, Dudzik,
  Chung, Choi, Powell, Ewalds, Georgiev, Oh, Horgan, Kroiss, Danihelka, Huang,
  Sifre, Cai, Agapiou, Jaderberg, Vezhnevets, Leblond, Pohlen, Dalibard,
  Budden, Sulsky, Molloy, Paine, Gulcehre, Wang, Pfaff, Wu, Ring, Yogatama,
  W{\"u}nsch, McKinney, Smith, Schaul, Lillicrap, Kavukcuoglu, Hassabis, Apps,
  and Silver]{vinyalsGrandmasterLevelStarCraft2019}
Vinyals, O., Babuschkin, I., Czarnecki, W.~M., Mathieu, M., Dudzik, A., Chung,
  J., Choi, D.~H., Powell, R., Ewalds, T., Georgiev, P., Oh, J., Horgan, D.,
  Kroiss, M., Danihelka, I., Huang, A., Sifre, L., Cai, T., Agapiou, J.~P.,
  Jaderberg, M., Vezhnevets, A.~S., Leblond, R., Pohlen, T., Dalibard, V.,
  Budden, D., Sulsky, Y., Molloy, J., Paine, T.~L., Gulcehre, C., Wang, Z.,
  Pfaff, T., Wu, Y., Ring, R., Yogatama, D., W{\"u}nsch, D., McKinney, K.,
  Smith, O., Schaul, T., Lillicrap, T., Kavukcuoglu, K., Hassabis, D., Apps,
  C., and Silver, D.
\newblock Grandmaster level in {{StarCraft II}} using multi-agent reinforcement
  learning.
\newblock \emph{Nature}, 575\penalty0 (7782):\penalty0 350--354, 2019.
\newblock ISSN 1476-4687.
\newblock \doi{10.1038/s41586-019-1724-z}.
\newblock URL \url{https://www.nature.com/articles/s41586-019-1724-z}.

\bibitem[Wang et~al.(2020)Wang, Gupta, Mahajan, Peng, Whiteson, and
  Zhang]{wang2020rode}
Wang, T., Gupta, T., Mahajan, A., Peng, B., Whiteson, S., and Zhang, C.
\newblock Rode: Learning roles to decompose multi-agent tasks.
\newblock \emph{arXiv preprint arXiv:2010.01523}, 2020.

\bibitem[Yu et~al.(2021)Yu, Velu, Vinitsky, Wang, Bayen, and
  Wu]{yu2021surprising}
Yu, C., Velu, A., Vinitsky, E., Wang, Y., Bayen, A., and Wu, Y.
\newblock The surprising effectiveness of ppo in cooperative, multi-agent
  games, 2021.

\bibitem[Zhang et~al.(2020)Zhang, Sodhani, Khetarpal, and
  Pineau]{zhang2020multi}
Zhang, A., Sodhani, S., Khetarpal, K., and Pineau, J.
\newblock Multi-task reinforcement learning as a hidden-parameter block mdp.
\newblock \emph{arXiv e-prints}, pp.\  arXiv--2007, 2020.

\end{thebibliography}
\bibliographystyle{icml2022}
\clearpage
\newpage
\onecolumn
\appendix
\section{Proofs}
\label{app:proofs}
\addtocounter{theorem}{-6}

\allowdisplaybreaks

\subsection{Generalisation between team compositions}\label{app:team_comp}
\begin{theorem}[Generalisation between team compositions]
Let team compositions $\mathcal{T}^x,\mathcal{T}^y \in \mathcal{C}^n$ with influence weights $a^x, a^y \in \Delta_{n-1}$, $s_{max} = \max_{s} ||W_R s||_1$
\am{this can be bound in terms of $W_R$ norm}, $V_{mid} = \frac{1}{2} \max_{s} V^{*}_{\mathcal{T}^y}(s)$, Then\footnote{for $\gamma \in (0, \frac{\sqrt{5}-1}{2})$ we can replace $\frac{1}{\gamma(1-\gamma)}$ by $\frac{1+\gamma}{1-\gamma}$}:
$$|V^{*}_{\mathcal{T}^x}- V^{*}_{\mathcal{T}^y}| \leq \frac{s_{max}+\gamma d V_{mid}}{\gamma(1-\gamma)} \Psi \text{, where}$$
$$\Psi = \Big[|\sum_i a^x_i(\mathcal{T}^x_i -\mathcal{T}^y_i)|_{\infty} + |\sum_i (a^x_i-a^y_i)\mathcal{T}^y_i |_{\infty}\Big]$$
\label{th:gen}
\begin{proof}
Let $\epsilon_R = \max_{s}|R_{\mathcal{T}^x}(s)-R_{\mathcal{T}^y}(s)|$ and $\epsilon_P = \max_{s,\mathbf{u}} 2\cdot D_{TV}\Big(P_{\mathcal{T}^x}(\cdot|s,\mathbf{u}), P_{\mathcal{T}^y}(\cdot|s,\mathbf{u})\Big)$ where $D_{TV}$ is the total variation distance.
We have that:
\begin{align}
    &|Q^{*}_{\mathcal{T}^x}(s, \mathbf{u}) - Q^{*}_{\mathcal{T}^y}(s, \mathbf{u})| \\ &=|R_{\mathcal{T}^x}(s)-R_{\mathcal{T}^y}(s)
    +\gamma\Big(\sum_{s'} P_{\mathcal{T}^x}(s'|s, \mathbf{u})\max_{u'}Q^{*}_{\mathcal{T}^x}(s', \mathbf{u'})
    -\sum_{s'} P_{\mathcal{T}^y}(s'|s, \mathbf{u})\max_{u'}Q^{*}_{\mathcal{T}^y}(s', \mathbf{u'}))\Big)|\\
    &\leq |R_{\mathcal{T}^x}(s)-R_{\mathcal{T}^y}(s)| +\gamma\Big\{|\sum_{s'} P_{\mathcal{T}^x}(s'|s, \mathbf{u})\Big[\max_{u'}Q^{*}_{\mathcal{T}^x}(s', \mathbf{u'})-\max_{u'}Q^{*}_{\mathcal{T}^y}(s', \mathbf{u'})\Big]|\\
    &\quad+|\sum_{s'} \Big[P_{\mathcal{T}^x}(s'|s, \mathbf{u})-P_{\mathcal{T}^y}(s'|s, \mathbf{u})\Big](\max_{u'}Q^{*}_{\mathcal{T}^y}(s', \mathbf{u'})-V_{mid})|\Big\} \text{\am{this is TV}}\\
    &\leq \epsilon_R 
    +\gamma\Big\{\sum_{s'} P_{\mathcal{T}^x}(s'|s, \mathbf{u})|\max_{u'}Q^{*}_{\mathcal{T}^x}(s', \mathbf{u'})-\max_{u'}Q^{*}_{\mathcal{T}^y}(s', \mathbf{u'})|+\sum_{s'} |P_{\mathcal{T}^x}(s'|s, \mathbf{u})-P_{\mathcal{T}^y}(s'|s, \mathbf{u})||\max_{u'}Q^{*}_{\mathcal{T}^y}(s', \mathbf{u'})-V_{mid}|\Big\}\\
    &\leq \epsilon_R 
    +\gamma\Big\{\sum_{s'} P_{\mathcal{T}^x}(s'|s, \mathbf{u})\max_{u'}|Q^{*}_{\mathcal{T}^x}(s', \mathbf{u'})-Q^{*}_{\mathcal{T}^y}(s', \mathbf{u'})|
    +2\cdot D_{TV}\Big(P_{\mathcal{T}^x}(s'|s, \mathbf{u}),P_{\mathcal{T}^y}(s'|s, \mathbf{u})\Big)V_{mid}\Big\}\\
    &\leq \epsilon_R 
    +\gamma\Big\{\max_{s',u'}|Q^{*}_{\mathcal{T}^x}(s', \mathbf{u'})-Q^{*}_{\mathcal{T}^y}(s', \mathbf{u'})|+\epsilon_P V_{mid}\Big\}
\end{align}
Next taking $\max$ w.r.t. ${s,u}$ of the above we get:
\begin{align}
   \max_{s,u} |Q^{*}_{\mathcal{T}^x}(s, \mathbf{u}) - Q^{*}_{\mathcal{T}^y}(s, \mathbf{u})| \leq  \frac{\epsilon_R+\gamma \epsilon_P V_{mid}}{1-\gamma}
\end{align}
We now bound the deviation quantities appearing above:
\begin{align}
    \epsilon_R &= \max_{s}|R_{\mathcal{T}^x}(s)-R_{\mathcal{T}^y}(s)|\\
    &=\max_{s}|\sum_{i=1}^n a^x_i \langle \mathcal{T}^x_i \cdot W_Rs\rangle-\sum_{i=1}^n a^y_i \langle \mathcal{T}^y_i \cdot W_Rs\rangle| \text{\am{note y's ind can be permuted}}\\
    &\leq \max_{s}\Big[|\sum_{i=1}^n a^x_i \langle (\mathcal{T}^x_i -\mathcal{T}^y_i) \cdot W_Rs\rangle| + |\sum_{i=1}^n (a^x_i-a^y_i) \langle \mathcal{T}^y_i \cdot W_Rs\rangle|\Big]\\
    &\leq\max_{s}\Big[ |\sum_i a^x_i(\mathcal{T}^x_i -\mathcal{T}^y_i)|_{\infty}|W_Rs|_1 + |\sum_i (a^x_i-a^y_i)\mathcal{T}^y_i |_{\infty}|W_Rs|_1\Big]\\
    &= s_{max}\Big[|\sum_i a^x_i(\mathcal{T}^x_i -\mathcal{T}^y_i)|_{\infty} + |\sum_i (a^x_i-a^y_i)\mathcal{T}^y_i |_{\infty}\Big]
\end{align}

Similarly,
\begin{align}
    \epsilon_P &= \max_{s,\mathbf{u}} 2\cdot D_{TV}\Big(P_{\mathcal{T}^x}(\cdot|s,\mathbf{u}), P_{\mathcal{T}^y}(\cdot|s,\mathbf{u})\Big)\\
    &= \max_{s,\mathbf{u}} \sum_{s'} |P_{\mathcal{T}^x}(s'|s, \mathbf{u})-P_{\mathcal{T}^y}(s'|s, \mathbf{u})|\\
    &=\max_{s,\mathbf{u}}\sum_{s'}|\sum_{i=1}^n a^x_i \langle \mathcal{T}^x_i \cdot W_P(s',s,\mathbf{u})\rangle-\sum_{i=1}^n a^y_i \langle \mathcal{T}^y_i \cdot W_P(s',s,\mathbf{u})\rangle|\\
    &\leq \max_{s,\mathbf{u}}\sum_{s'}\Big[|\sum_{i=1}^n a^x_i \langle (\mathcal{T}^x_i -\mathcal{T}^y_i) \cdot W_P(s',s,\mathbf{u})\rangle| + |\sum_{i=1}^n (a^x_i-a^y_i) \langle \mathcal{T}^y_i \cdot W_P(s',s,\mathbf{u})\rangle|\Big]\\
    &\leq\max_{s,\mathbf{u}}\sum_{s'}\Big[ |\sum_i a^x_i(\mathcal{T}^x_i -\mathcal{T}^y_i)|_{\infty}|W_P(s',s,\mathbf{u})|_1 + |\sum_i (a^x_i-a^y_i)\mathcal{T}^y_i |_{\infty}|W_P(s',s,\mathbf{u})|_1\Big]\\
    &= \Big[|\sum_i a^x_i(\mathcal{T}^x_i -\mathcal{T}^y_i)|_{\infty} + |\sum_i (a^x_i-a^y_i)\mathcal{T}^y_i |_{\infty}\Big]\max_{s,\mathbf{u}}\sum_{s'}|W_P(s',s,\mathbf{u})|_1\\    
    &=d\Big[|\sum_i a^x_i(\mathcal{T}^x_i -\mathcal{T}^y_i)|_{\infty} + |\sum_i (a^x_i-a^y_i)\mathcal{T}^y_i |_{\infty}\Big]
\end{align}

Thus, we get:
\begin{align}
    |Q^{*}_{\mathcal{T}^x}(s, \mathbf{u}) - Q^{*}_{\mathcal{T}^y}(s, \mathbf{u})|\leq \frac{s_{max}+\gamma d V_{mid}}{1-\gamma} \Big[|\sum_i a^x_i(\mathcal{T}^x_i -\mathcal{T}^y_i)|_{\infty} + |\sum_i (a^x_i-a^y_i)\mathcal{T}^y_i |_{\infty}\Big]
\end{align} 
Finally we get the value difference bound by considering a dummy state $s^{\#}$ which always transitions according to $\rho$ and then using the Bellman equation. (Note that for $\gamma \in (0, \frac{\sqrt{5}-1}{2})$ we can replace $\frac{1}{\gamma(1-\gamma)}$ by $\frac{1+\gamma}{1-\gamma}$ for a tighter bound without considering a dummy start state)
\end{proof}
\end{theorem}

\begin{corollary}[Change in optimal value as a result of agent substitution] Let $\mathcal{T} \in \mathcal{C}^n$ be a team composition with influence weights $a \in \Delta_{n-1}$. If agent $i$ is substituted with $i'$ keeping $a_i$ unchanged such that $|\mathcal{T}_{i'}-\mathcal{T}_{i}|_\infty\leq \epsilon_C$ then the new team ($\mathcal{T}'$) optimal value follows: 
\begin{align}
    |V^{*}_{\mathcal{T}'}- V^{*}_{\mathcal{T}}|\leq \frac{(s_{max}+\gamma d V_{mid})a_i\epsilon_C}{\gamma(1-\gamma)}
\end{align} 
\end{corollary}
\begin{proof}
Applying \cref{th:gen} on original task and a new task with same influence weights and agent $i$ capability replaced with $\mathcal{T}_{i'}$ immediately gives the result.
\end{proof}

\subsection{Transfer of optimal policy}\label{app:transfer}
\begin{theorem}[Transfer of optimal policy] Let $\mathcal{T}^x,\mathcal{T}^y \in \mathcal{C}^n$, $a^x, a^y \in \Delta_{n-1}$, $s_{max} = \max_{s} ||W_R s||_1$, $V_{mid} = \frac{1}{2}\max_{ s} V^{*}_{\mathcal{T}^y}(s)$. Let $\pi_y^*$ be the optimal policy for the team composed of agents with capabilities $\mathcal{T}^y$ and influence weights $a^y$. Then:
$$V^{*}_{\mathcal{T}^x} - V^{\pi_{y}^*}_{\mathcal{T}^x} \leq 2\frac{s_{max}+\gamma d V_{mid}}{\gamma(1-\gamma)} \Psi,$$
where $\Psi$ is defined as in \cref{eqpsidef}.
\label{th:trans}
\end{theorem}
\begin{proof}
We have that:
\begin{align}
    Q^{*}_{\mathcal{T}^x}(s, \mathbf{u}) - Q^{\pi_{y}^*}_{\mathcal{T}^x}(s, \mathbf{u}) &\leq |Q^{*}_{\mathcal{T}^x}(s, \mathbf{u}) - Q^{*}_{\mathcal{T}^y}(s, \mathbf{u})| + |Q^{*}_{\mathcal{T}^y}(s, \mathbf{u}) - Q^{\pi_{y}^*}_{\mathcal{T}^x}(s, \mathbf{u})|
    \label{eqtdecomp}
\end{align}
The first term on the RHS of \cref{eqtdecomp} is taken care of by \cref{th:gen}. We now focus on the second term:
\begin{align}
    &|Q^{*}_{\mathcal{T}^y}(s, \mathbf{u}) - Q^{\pi_{y}^*}_{\mathcal{T}^x}(s, \mathbf{u})| \\ &=|R_{\mathcal{T}^y}(s)-R_{\mathcal{T}^x}(s)
    +\gamma\Big(\sum_{s'} P_{\mathcal{T}^y}(s'|s, \mathbf{u})\max_{u'}Q^{*}_{\mathcal{T}^y}(s', \mathbf{u'})
    -\sum_{s'} P_{\mathcal{T}^x}(s'|s, \mathbf{u})Q^{\pi_{y}^*}_{\mathcal{T}^x}(s', \pi_{y}^*(\mathbf{u'}))\Big)|\\
    &\leq \epsilon_R +\gamma\Big\{|\sum_{s'} P_{\mathcal{T}^x}(s'|s, \mathbf{u})\Big[\max_{u'}Q^{*}_{\mathcal{T}^y}(s', \mathbf{u'})-Q^{\pi_{y}^*}_{\mathcal{T}^x}(s', \pi_{y}^*(\mathbf{u'})\Big]| \text{\am{note $\pi_{y}^*$ is optimal for $Q^{*}_{\mathcal{T}^y}$}}\\
    &\quad+|\sum_{s'} \Big[P_{\mathcal{T}^y}(s'|s, \mathbf{u})-P_{\mathcal{T}^x}(s'|s, \mathbf{u})\Big](\max_{u'}Q^{*}_{\mathcal{T}^y}(s', \mathbf{u'})-V_{mid})|\Big\}\\
    &\leq \epsilon_R +\gamma\Big\{\max_{s', u'}|Q^{*}_{\mathcal{T}^y}(s', \mathbf{u'})-Q^{\pi_{y}^*}_{\mathcal{T}^x}(s', \pi_{y}^*(\mathbf{u'})| +\epsilon_P V_{mid}\Big\}
\end{align}
Once again, taking $\max$ w.r.t. ${s,\textbf{u}}$ of the above we get:
\begin{align}
   \max_{s,u} |Q^{*}_{\mathcal{T}^y}(s, \mathbf{u}) - Q^{\pi_{y}^*}_{\mathcal{T}^x}(s, \mathbf{u})| \leq  \frac{\epsilon_R+\gamma \epsilon_P V_{mid}}{1-\gamma}
\end{align}
Substituting for deviation expressions and using \cref{th:gen} in \cref{eqtdecomp} we get:
\begin{align}
    |Q^{*}_{\mathcal{T}^x}(s, \mathbf{u}) - Q^{\pi_{y}^*}_{\mathcal{T}^x}(s, \mathbf{u})|\leq 2\frac{s_{max}+\gamma d V_{mid}}{1-\gamma} \Big[|\sum_i a^x_i(\mathcal{T}^x_i -\mathcal{T}^y_i)|_{\infty} + |\sum_i (a^x_i-a^y_i)\mathcal{T}^y_i |_{\infty}\Big]
\end{align} 
Note the absolute on LHS above can be dropped as $Q^{*}_{\mathcal{T}^x}$ is optimal. Finally using the same technique as above for \cref{th:gen} we get the statement of the theorem.
\end{proof}

\begin{corollary}[Out of distribution performance]
Let $\mathcal{T}\notin Sup(\mathcal{M})$ be an out of distribution task, we then have that the performance of the absolute oracle policy on $\mathcal{T}$ satisfies:
$$V^{*}_{\mathcal{T}} - V^{\pi^*_{\mathcal{M}}}_{\mathcal{T}} \leq 2\frac{s_{max}+\gamma d V_{mid}}{\gamma(1-\gamma)}d_{a}(\mathcal{T}, Sup(\mathcal{M})),$$
\end{corollary}
\begin{proof}
For any task that belongs to $\argmin_{\mathcal{T}^l \in Sup(\mathcal{M})} d_{a}(\mathcal{T}^l, \mathcal{T})$, we have by application of \cref{th:trans} that the result immediately holds given definition of $\pi_{\mathcal{M}}^*$.
\end{proof}

\subsection{Population decrease}\label{app:pop_decrease}

\begin{theorem}[Population decrease bound] For the team composition $\mathcal{T} \in \mathcal{C}^n$ with influence weights $a \in \Delta_{n-1}$. If agent $n$ is eliminated  followed by a re-normalization of influence weights, we have that for the remaining team ($\mathcal{T}^- \triangleq (\mathcal{T})_{i=1}^{n-1}$):
\begin{align}
    |V^{*}_{\mathcal{T}^-}- V^{*}_{\mathcal{T}}|\leq \frac{a_n(s_{max}+\gamma d V_{mid})}{\gamma(1-\gamma)}\Big|\sum_{i=1}^{n-1} \frac{a_i\mathcal{T}_i }{1-a_n} - \mathcal{T}_n \Big|_\infty
\end{align} 
\label{th:elim}
\end{theorem}
\begin{proof} We use \cref{th:gen} with influence weights $(a_i)_1^n$ and $(\lambda\cdot a_i: i=1..n-1, a_n=0)$ where $\lambda = \frac{1}{1-a_n}$
\end{proof}

\begin{corollary}[Population increase bound] For the team composition $\mathcal{T} \in \mathcal{C}^n$ with influence weights $a \in \Delta_{n-1}$. If agent $n+1$ is added with capability $\mathcal{T}_{n+1}$ and weight $a_{n+1}$ (other weights scaled down by $\lambda = 1-a_{n+1}$) we have that for the new team ($\mathcal{T}^+ \triangleq (\mathcal{T}_1.. \mathcal{T}_n, \mathcal{T}_{n+1})$):
\begin{align}
    |V^{*}_{\mathcal{T}^+}- V^{*}_{\mathcal{T}}|\leq \frac{a_{n+1}(s_{max}+\gamma d V_{mid})}{\gamma(1-\gamma)}\Big|\sum_{i=1}^{n} a_i\mathcal{T}_i  - \mathcal{T}_{n+1} \Big|_\infty
\end{align}
\end{corollary}
\begin{proof}
Consider the team compositions $\mathcal{T}^x = (\mathcal{T}_1..\mathcal{T}_n,0)$ with influence weights = $(a_1..a_n, 0)$ and $\mathcal{T}^y = (\mathcal{T}_1..\mathcal{T}_n,\mathcal{T}_{n+1})$ with influence weights = $(\lambda a_1..\lambda a_n, a_{n+1})$ where $\lambda = 1-a_{n+1}$, we have that:
\begin{align}
\Psi &= \Big[|\sum_i a^x_i(\mathcal{T}^x_i -\mathcal{T}^y_i)|_{\infty} + |\sum_i (a^x_i-a^y_i)\mathcal{T}^y_i |_{\infty}\Big] \\
&= |\sum_{i=1}^n (1-\lambda)a_i\mathcal{T}^y_i - a_{n+1}\mathcal{T}^y_{n+1} |_{\infty}\\
&=  a_{n+1}|\sum_{i=1}^n a_i\mathcal{T}^y_i - \mathcal{T}^y_{n+1} |_{\infty}
\end{align}
which on applying \cref{th:gen} yields the result.
\end{proof}

\subsection{Approximate $\hat\epsilon_R$,$\hat\epsilon_P$ dynamics}\label{app:approximate}

\begin{theorem}[Approximate $\hat\epsilon_R$,$\hat\epsilon_P$ dynamics]
Let $\mathcal{T}^x,\mathcal{T}^y \in \mathcal{C}^n$, $a^x, a^y \in \Delta_{n-1}$ and the dynamics be only approximately linear so that $|R_{\mathcal{T}}(s)- \sum_{i=1}^n a_i \langle c_i \cdot W_Rs\rangle| \leq \hat\epsilon_R$ and $|P_{\mathcal{T}}(s'|s,\mathbf{u}) - \sum_{i=1}^n a_i \langle c_i \cdot  W_P(s',s,\mathbf{u}) \rangle|\leq \hat\epsilon_P$.
Then:
$$|V^{*}_{\mathcal{T}^x}- V^{*}_{\mathcal{T}^y}|\leq \frac{s_{max}+\gamma d V_{mid}}{\gamma(1-\gamma)} \Psi + \frac{2(\hat\epsilon_R+\gamma \hat\epsilon_P V_{mid})}{\gamma(1-\gamma)},$$
where $\Psi$ is defined as in \cref{eqpsidef}.
\label{th:gen_err}
\end{theorem}
\begin{proof}
We begin as in proof of \cref{th:gen} to get:
\begin{align}
   \max_{s,u} |Q^{*}_{\mathcal{T}^x}(s, \mathbf{u}) - Q^{*}_{\mathcal{T}^y}(s, \mathbf{u})| \leq  \frac{\epsilon_R+\gamma \epsilon_P V_{mid}}{1-\gamma}
\end{align}
Next we apply the corrections to the relative differences:
\begin{align}
    \epsilon_R &= \max_{s}|R_{\mathcal{T}^x}(s)-R_{\mathcal{T}^y}(s)|\\
    &\leq\max_{s}\Big[|R_{\mathcal{T}^x}(s)- \sum_{i=1}^n a_i^x \langle \mathcal{T}^x_i\cdot W_Rs\rangle|+|\sum_{i=1}^n a^x_i \langle \mathcal{T}^x_i \cdot W_Rs\rangle-\sum_{i=1}^n a^y_i \langle \mathcal{T}^y_i \cdot W_Rs\rangle|+|R_{\mathcal{T}^y}(s)- \sum_{i=1}^n a_i^y \langle \mathcal{T}^y_i\cdot W_Rs\rangle| \Big] \\
    &\leq 2\hat\epsilon_R+\max_{s}\Big[|\sum_{i=1}^n a^x_i \langle (\mathcal{T}^x_i -\mathcal{T}^y_i) \cdot W_Rs\rangle| + |\sum_{i=1}^n (a^x_i-a^y_i) \langle \mathcal{T}^y_i \cdot W_Rs\rangle|\Big]\\
    &\leq 2\hat\epsilon_R+\max_{s}\Big[ |\sum_i a^x_i(\mathcal{T}^x_i -\mathcal{T}^y_i)|_{\infty}|W_Rs|_1 + |\sum_i (a^x_i-a^y_i)\mathcal{T}^y_i |_{\infty}|W_Rs|_1\Big]\\
    &= 2\hat\epsilon_R+s_{max}\Big[|\sum_i a^x_i(\mathcal{T}^x_i -\mathcal{T}^y_i)|_{\infty} + |\sum_i (a^x_i-a^y_i)\mathcal{T}^y_i |_{\infty}\Big]
\end{align}
Proceeding similarly with the transition probabilities we get the desired result.
\end{proof}

\subsection{Error from estimation of capabilities}\label{app:error}
\begin{theorem}[Error from estimation of capabilities]
For the team composition $\mathcal{T} \in \mathcal{C}^n$ with influence weights $a \in \Delta_{n-1}$. If the agent capabilities are inaccurately inferred as $\hat{\mathcal{T}}$ with $\max_i|\mathcal{T}_i- \hat{\mathcal{T}}_i|_\infty \leq \epsilon_{\mathcal{T}}$ and agents learn the inexact policy $\hat{\pi}^*$ then:
$$|V^{*}_{\mathcal{T}}- V^{\hat{\pi}^*}_{\mathcal{T}}| \leq \frac{2\epsilon_{\mathcal{T}}(s_{max}+\gamma d V_{mid})}{\gamma(1-\gamma)}$$ where $V_{mid} = \frac{1}{2} \max_s V_{\hat{\mathcal{T}}}^*(s)$
\label{th:err_est}
\end{theorem}
\begin{proof}
We have that for the actual and inferred team compositions with same influence weights:
\begin{align}
\Psi &= \Big[|\sum_i a_i(\mathcal{T}_i -\hat{\mathcal{T}}_i)|_{\infty} + |\sum_i (a_i-a_i)\hat{\mathcal{T}}_i |_{\infty}\Big] \\
&= |\sum_i a_i(\mathcal{T}_i -\hat{\mathcal{T}}_i)|_{\infty}\\
&\leq \sum_i a_i|\mathcal{T}_i -\hat{\mathcal{T}}_i|_{\infty}\\
&\leq \sum_i a_i \epsilon_{\mathcal{T}}\\
&=\epsilon_{\mathcal{T}}
\end{align}
Now applying \cref{th:trans} gives the result
\end{proof}

\subsection{Extending to Lipschitz rewards}
\label{app:lip_rews}
We demonstrate how to extend the results in \cref{sec:analysis} to Lipschitz function of capabilities. For brevity  we consider only the setting where the rewards vary with capabilities.
Thus, for the reward function form $R_{\mathcal{T}}(s) = \langle f(\mathcal{T}) \cdot W_Rs \rangle$ where $ f(\mathcal{T})$ is $L_i$ Lipschitz with respect to the capability $\mathcal{T}_i$ for $i \in \mathcal{A}$ for the $|\cdot|_\infty$ norm. We get that for two different team compositions $\mathcal{T}^x, \mathcal{T}^y$
\begin{align}
\epsilon_R &= \max_{s}|R_{\mathcal{T}^x}(s)-R_{\mathcal{T}^y}(s)|\\
&=\max_{s}| \langle f(\mathcal{T}^x) \cdot W_Rs \rangle-\langle f(\mathcal{T}^y) \cdot W_Rs \rangle|\\
&=\max_{s}|\sum_{i=1}^{n} \langle f(\mathcal{T}^i) \cdot W_Rs \rangle-\langle f(\mathcal{T}^{i+1}) \cdot W_Rs \rangle|\\
&\leq\max_{s}\sum_{i=1}^{n} |\langle f(\mathcal{T}^i) \cdot W_Rs \rangle-\langle f(\mathcal{T}^{i+1}) \cdot W_Rs \rangle|\\
&\leq\max_{s}\sum_{i=1}^{n} |\langle f(\mathcal{T}^i) \cdot W_Rs \rangle-\langle f(\mathcal{T}^{i+1}) \cdot W_Rs \rangle|\\
&\leq\max_{s}\sum_{i=1}^{n} | f(\mathcal{T}^i) - f(\mathcal{T}^{i+1})|_\infty |W_Rs|_1\\
&\leq s_{max}\sum_{i=1}^{n} L_i| \mathcal{T}^x_i - \mathcal{T}^y_i|_\infty\\
\end{align}
Where $\mathcal{T}^i$ was the sequence satisfying $\mathcal{T}^1=\mathcal{T}^x$ and $\mathcal{T}^{n+1}=\mathcal{T}^y$ and changing $\mathcal{T}^x$ one index at a time.
We have thus proved that:
\begin{theorem}
For rewards $L_i$ Lipschitz in the capabilities with respect to $|\cdot|_\infty$ norm, the difference in optimal values between team compositions $\mathcal{T}^x, \mathcal{T}^y$ satisfy:
$$|V^{*}_{\mathcal{T}^x}- V^{*}_{\mathcal{T}^y}|\leq \frac{s_{max}\sum_{i=1}^{n} L_i| \mathcal{T}^x_i - \mathcal{T}^y_i|_\infty}{\gamma(1-\gamma)}$$
\end{theorem}
\subsection{General dependence of rewards on capabilities:}
\label{app:gen_rews}
We now consider the dependence of rewards on the capabilities in the most general form. For this, we introduce the notion of $(\alpha, k)$-rewards where $\alpha \geq 0, k \in \mathbb{N}$.
\begin{align}
    R_{\mathcal{T}}(s) &=  \Big\langle \sum_{k_i\in \mathbb{N}, \sum k_i \leq k} a_{k_1..k_n}\Pi_{i=1}^n c_i^{k_i} \cdot W_Rs \Big\rangle \label{eqgenr}
\end{align}
where $\mathbb{N}$ are non negative integers, $|a_{k_1..k_n}|\leq \alpha$ and $c_i^{k_i}$ represents element-wise exponentiation. \am{k here clashed with state dimension, chose different}. Rewards in \cref{eqlinr} can be seen as a special case belonging to \cref{eqgenr} the choice $\alpha,k =1$. Similarly the union $\cup_{\alpha\geq 0, k\in\mathbb{N}} (\alpha, k)\text{-rewards}$ cover all possible reward dependencies on capabilities. We have further relaxed the assumption of influence weights belonging to a simplex here and replaced it with individual bounds on the power series coefficients here.  We next see that for this scenario, even a small change in the capability of a single agent can shift the rewards massively. \am{give max deviation bound for $\delta$ shift, is a non-trivial lower bound possible? Also to see is polynomial kernel based results and opt trans} Let the capability of agent $i$ be changed from $\mathcal{T}_i$ to $\mathcal{T}_{i'}$ such that $|\mathcal{T}_i-\mathcal{T}_{i'}|_{\infty} \leq \delta$. Then we have 
\begin{lemma}
For substitution $\mathcal{T}_i$ to $\mathcal{T}_{i'}$ such that $|\mathcal{T}_i-\mathcal{T}_{i'}|_{\infty} \leq \delta$ under the $(\alpha, k)$-rewards setting we have that 
\begin{align}
    \epsilon_R &= \max_{s \in S} \Big|\Big\langle \sum_{k_i\in \mathbb{N}, \sum k_i \leq k} a_{k_1..k_n}\Pi_{j\neq i} \mathcal{T}_j^{k_j} (\mathcal{T}_i^{k_i} - \mathcal{T}_{i'}^{k_i}) \cdot W_Rs \Big\rangle\Big| \\ 
    & \leq \max_{s \in S} \Big|\sum_{k_i\in \mathbb{N}, \sum k_i \leq k} a_{k_1..k_n}\Pi_{j\neq i} \mathcal{T}_j^{k_j} (\mathcal{T}_i^{k_i} - \mathcal{T}_{i'}^{k_i})\Big|_{\infty}\Big|W_Rs\Big|_1 \\
    & \leq \alpha s_{max} \sum_{j=0}^k \sum_{l=1}^j  \binom{l}{j}l|\mathcal{T}_i^{k_i} - \mathcal{T}_{i'}^{k_i}|_{\infty}\\
    &\leq \alpha \delta s_{max} \sum_{j=0}^k j 2^{j-1} = \mathcal{O}(\alpha \delta s_{max}k 2^k)
\end{align}
\end{lemma}
While this is not a lower bound, the above still suggests that even a small change in the capability of an agent can cause the rewards to change by a lot, hence it is natural to expect that generalization becomes harder as the problem start showing the needle in the haystack phenomenon where only the \textit{right} \textit{combination} of capabilities gives a large optimal value. 	
\section{Experimental Setup}\label{app:experiments}

\subsection{Environments}\label{app:envs}
\subsubsection{Fruit Forage}\label{app:fruit}

We use the fruit forage task on a grid world to empirically demonstrate the generalisation bounds in \cref{sec:analysis}. On a $k\times k$ grid world we have $n$ agents and $d$ types of fruit trees. For each agent $i$, $\mathcal{T}_i(j), j \in \{1..d\}$ represents the utility of fruit $j$ for agent $i$. The state vector is appended with the $d$ dimensional binary vector representing whether each of the tree types was foraged at a given time step. 
We define three team compositions as follows:
\begin{itemize}
    \item {$T_x$: [[0.05, 0.1, 0.6, 2.8], [0.05, 0.1, 2.1, 0.8], [0.05, 0.1, 1.8, 1.2], [0.05, 0.1, 0.9, 2.4]]}
    \item {$T_y$: [[0.7, 0.4, 0.15, 0.2], [0.2, 1.4, 0.15, 0.2], [0.3, 1.2, 0.15, 0.2], [0.6, 0.6, 0.15, 0.2]]} 
    \item {$T_z$: [[0.1, 0.3, 0.6, 0.0], [0.4, 0.1, 0.5, 0.0], [0.05, 0.06, 0.89, 0.0], [0.0, 0.0, 0.0, 1.0]]}
\end{itemize}

For proving bounds on Theorem-1, we compare the mean test returns achieved on tasks $T_x$ and $T_y$ using $V^{\star}_{T_x} - V^{\star}_{T_y}$. For Theorem-2, we compare the mean test returns achieved on tasks $T_x$ and optimal policies of task $T_y$ evaluated on task $T_x$ i.e. $V^{\star}_{T_x} - V^{\pi^\star_{T_y}}_{T_x}$. Finally, for Theorem-3, we compare the mean test returns achieved on tasks $T_z$ and optimal policies of task $T_z$ evaluated on task $T_z$ but removing the last agent i.e. $V^{\star}_{T_{z-}} - V^{\star}_{T_z}$.

\subsubsection{Predator Prey}\label{app:pred_prey}
We consider a complicated partially observable predator-prey (PP) task in an $8 \times 8$ grid involving four agents (predators) and four prey that is designed to test coordination between agents. Specifically, each predator has a parameter describing the hit point damage it can cause the prey. Similarly, the prey comes with variations in health. For example, a prey with a capability of 5 can only be caught if the total capability of agents taking the capture action simultaneously on it have capabilities $\geq$ 5 (such as [1,1,3]), otherwise, the whole team receives a penalty $p$. On successful capture, agents get a reward of +1. Once prey is captured, another prey is spawned at a random location. Therefore, agents have to collaborate and capture as many preys as possible within 100 time steps.

Each agent can take 6 actions i.e. move in one of the 4 directions (Up, Left, Down, Right), remain still (no-op), or try to catch (capture) any adjacent prey. The prey moves around in the grid with a probability of 0.7 and remains still at its position with the probability of 0.3. Impossible actions for both agents and prey are marked unavailable, for eg. moving into an occupied cell or trying to take a capture action with no adjacent prey.

In this domain, we test for two types of generalization:
(1) novel team composition where test tasks contain a team composition which has not been encountered during training (PP Unseen Team in Figure \ref{fig:pp_v1_v2}), and second, (2) test tasks where novel team compositions can also have agent types with capabilities not encountered during training (PP Unseen Team, Agent in Figure \ref{fig:pp_v1_v2}).

For (PP Unseen Team), we train on preys with capabilities [2,2,2,3], and agents with capabilities [2,3,2,3],[1,2,1,2], thereby having agent teams with total hit points of 10 and 6 respectively. We also train on two separate penalties $p$ for miscoordination i.e. $p \in \{0.0, -0.008\}$, this helps inject additional stochasticity in the environment as the agents don't know the penalty value. For test tasks, we create novel team compositions not encountered during training i.e. agents with capabilities [1,1,2,3],[1,1,1,3] having total hit points of 7 and 6 respectively.

For (PP Unseen Team, Agent) we train on preys with capabilities [1,2,3,4], and agents with capabilities [1, 2, 2, 3], [1, 1, 2, 2], [1, 3, 2, 1], thereby having agent teams with total hit points of 8, 6 and 7 respectively. We also train on two separate penalties $p$ for miscoordination i.e. $p \in \{0.0, -0.008\}$. For test tasks, we create novel team compositions with an unseen agent of capability 4 not encountered during training i.e. agents with capabilities [1, 1, 1, 4], [1, 1, 3, 4], [1, 1, 2, 4] having total hit points of 7, 9, and 8 respectively.

\textbf{Experimental Setup:} For (PP Unseen Team, and PP Unseen Team, Agent), we show the average difference in performance across all test tasks between scenario when capability information is included (oracle, represented by (C) in the plots) and when it's not (naive) for each method. 

For testing the generalization gap in (PP Unseen Team), we show the difference in returns achieved by training task [1,2,1,2] (hit point 6) and test task [1,1,1,3] (hit point 6). For testing the generalization gap in (PP Unseen Team, Agent), we show the difference in returns achieved by training task [1,3,2,1] (hit point 7) and test task [1,1,1,4] (hit point 7) with a new agent of capability 4. All PP experiments are based on 8 seeds.

\subsubsection{StarCraft II}\label{app:smac}

We use the standard set of actions and global state information included as part of the SMAC benchmark \cite{samvelyan2019starcraft}. The sight range of the agent units has been increased to the fully observable setting. In the oracle mode, agent capabilities are included as part of individual observations. Each agent always observes its own capabilities. Furthermore, capabilities are always included in the global state.

\texttt{10\_Terran} and \texttt{10\_Terran\_Hard} environment includes Marine, Maradeur, and Medivac units.  
\texttt{10\_Protoss} and \texttt{10\_Protoss\_Hard} environments feature Stalker, Zealot, and Colossus units.
\texttt{10\_Zerg} and \texttt{10\_Zerg\_Hard} environments include Zergling, Hydralisk and Baneling units.

In \texttt{Accuracy} and \texttt{Health} tasks, specific values reduced from full unit capabilities are chosen to be  equivalent to a loss of a single teammate. For example, if there three agents, their accuracy could be set
to $0.75$, $0.75$ and $0.5$ given that $(1 - 0.5) + (1 - 0.75) + (1 - 0.75) = 1$. Consequently, the overall reduction in accuracy
would be roughly equivalent to losing one ally unit.
This was chosen to ensure that the difficulty of the tasks was not too high. 

All SMAC experiments are based on 5 seeds.

\cref{tab:10terran}, \ref{tab:10zerg}, and \ref{tab:10protoss} describe the training and evaluation distributions used in unit type swapping tasks.

\begin{table}[H]
    \centering
    \caption{Team formations in \texttt{Terran} tasks}\label{tab:10terran}
    \begin{tabular}{ll}
        \toprule
\texttt{10\_Terran} & \texttt{10\_Terran\_Hard} \\
        \midrule
\textbf{Training} & \textbf{Training} \\
1 marine \& 9 marauders                    & 1 marine \& 9 marauders     \\
3 marines \& 7 marauders                   & 2 marines \& 8 marauders    \\
4 marines \& 6 marauders                   & 3 marines \& 7 marauders    \\
5 marines \& 5 marauders                   & 4 marines \& 6 marauders    \\
6 marines \& 4 marauders                   & 5 marines \& 5 marauders    \\
8 marines \& 2 marauders                   & 6 marines \& 4 marauders    \\
9 marines \& 1 marauder                    & 7 marines \& 3 marauders    \\
5 marauders \& 5 medivacs                  & 8 marines \& 2 marauders    \\
7 marauders \& 3 medivacs                  & 9 marines \& 1 marauder     \\
9 marauders \& 1 medivac                   & 5 marauders \& 5 medivacs   \\
7 marines \& 3 medivacs                    & 6 marauders \& 4 medivacs   \\
8 marines \& 2 medivacs                    & 7 marauders \& 3 medivacs   \\
9 marines \& 1 medivac                     & 8 marauders \& 2 medivacs   \\
10 marines                                  & 9 marauders \& 1 medivac    \\
10 marauders                                & 7 marines \& 3 medivacs     \\
8 marines \& 1 marauder \& 1 medivac      & 8 marines \& 2 medivacs     \\
1 marine \& 8 marauders \& 1 medivac      & 9 marines \& 1 medivac      \\
5 marines \& 3 marauders \& 2 medivacs    & \textbf{Testing}             \\
2 marines \& 7 marauders \& 1 medivac     & 10 marines                   \\
6 marines \& 2 marauders \& 2 medivacs    & 10 marauders                 \\
2 marines \& 6 marauders \& 2 medivacs    & 8 marines \& 1 marauder \& 1 medivac \\
4 marines \& 4 marauders \& 2 medivacs    & 1 marine \& 8 marauders \& 1 medivac   \\
\textbf{Testing}                            & 5 marines \& 3 marauders \& 2 medivacs \\
2 marines \& 8 marauders                   & 3 marines \& 5 marauders \& 2 medivacs \\
7 marines \& 3 marauders                   & 4 marines \& 3 marauders \& 3 medivacs \\
6 marauders \& 4 medivacs                  & 3 marines \& 4 marauders \& 3 medivacs \\
8 marauders \& 2 medivacs                  & 7 marines \& 2 marauders \& 1 medivac  \\
3 marines \& 5 marauders \& 2 medivacs    & 2 marines \& 7 marauders \& 1 medivac  \\
4 marines \& 3 marauders \& 3 medivacs    & 6 marines \& 2 marauders \& 2 medivacs \\
3 marines \& 4 marauders \& 3 medivacs    & 2 marines \& 6 marauders \& 2 medivacs \\
7 marines \& 2 marauders \& 1 medivac     & 4 marines \& 4 marauders \& 2 medivacs \\
        \bottomrule
    \end{tabular}
\end{table}

\begin{table}[H]
    \centering
    \caption{Team formations in \texttt{Zerg} tasks}\label{tab:10zerg}
    \begin{tabular}{ll}
        \toprule
\texttt{10\_Zerg} & \texttt{10\_Zerg\_Hard} \\
        \midrule
\textbf{Training} & \textbf{Training} \\
1 zergling \& 9 hydralisks                        &     1 zergling \& 9 hydralisks     \\
2 zerglings \& 8 hydralisks                      &       2 zerglings \& 8 hydralisks   \\
4 zerglings \& 6 hydralisks                      &       3 zerglings \& 7 hydralisks   \\
5 zerglings \& 5 hydralisks                      &       4 zerglings \& 6 hydralisks   \\
6 zerglings \& 4 hydralisks                      &       5 zerglings \& 5 hydralisks   \\
7 zerglings \& 3 hydralisks                      &       6 zerglings \& 4 hydralisks   \\
9 zerglings \& 1 hydralisk                       &       7 zerglings \& 3 hydralisks   \\
4 hydralisks \& 6 banelings                      &       8 zerglings \& 2 hydralisks   \\
5 hydralisks \& 5 banelings                      &       9 zerglings \& 1 hydralisk    \\
6 hydralisks \& 4 banelings                      &       4 hydralisks \& 6 banelings   \\
8 hydralisks \& 2 banelings                      &       5 hydralisks \& 5 banelings   \\
9 hydralisks \& 1 baneling                       &       6 hydralisks \& 4 banelings   \\
4 zerglings \& 6 banelings                       &       7 hydralisks \& 3 banelings   \\
6 zerglings \& 4 banelings                       &       8 hydralisks \& 2 banelings \\
7 zerglings \& 3 banelings                       &       9 hydralisks \& 1 baneling  \\
8 zerglings \& 2 banelings                       &       4 zerglings \& 6 banelings  \\
10 zerglings                                      &      5 zerglings \& 5 banelings  \\
10 hydralisks                                     &      6 zerglings \& 4 banelings  \\
10 banelings                                      &      7 zerglings \& 3 banelings  \\
8 zerglings \& 1 hydralisk \& 1 baneling        &        8 zerglings \& 2 banelings  \\
1 zergling \& 8 hydralisks \& 1 baneling        &        9 zerglings \& 1 baneling   \\
7 zerglings \& 2 hydralisks \& 1 baneling       &        \textbf{Testing}             \\
2 zerglings \& 7 hydralisks \& 1 baneling       &        10 zerglings                \\
5 zerglings \& 3 hydralisks \& 2 banelings      &        10 hydralisks               \\
3 zerglings \& 5 hydralisks \& 2 banelings      &        10 banelings                \\
4 zerglings \& 4 hydralisks \& 2 banelings      &        8 zerglings \& 1 hydralisk \& 1 baneling    \\
3 zerglings \& 4 hydralisks \& 3 banelings      &        1 zergling \& 8 hydralisks \& 1 baneling    \\
\textbf{Testing}                                   &      7 zerglings \& 2 hydralisks \& 1 baneling   \\
3 zerglings \& 7 hydralisks                      &       2 zerglings \& 7 hydralisks \& 1 baneling   \\
8 zerglings \& 2 hydralisks                      &       6 zerglings \& 2 hydralisks \& 2 banelings  \\
7 hydralisks \& 3 banelings                      &       2 zerglings \& 6 hydralisks \& 2 banelings  \\
5 zerglings \& 5 banelings                       &       5 zerglings \& 3 hydralisks \& 2 banelings  \\
9 zerglings \& 1 baneling                        &       3 zerglings \& 5 hydralisks \& 2 banelings  \\
6 zerglings \& 2 hydralisks \& 2 banelings      &        4 zerglings \& 4 hydralisks \& 2 banelings  \\
4 zerglings \& 3 hydralisks \& 3 banelings      &        4 zerglings \& 3 hydralisks \& 3 banelings  \\
2 zerglings \& 6 hydralisks \& 2 banelings      &        3 zerglings \& 4 hydralisks \& 3 banelings  \\
        \bottomrule
    \end{tabular}
\end{table}

\begin{table}[H]
    \centering
    \caption{Team formations in \texttt{Protoss} tasks}\label{tab:10protoss}
    \begin{tabular}{ll}
        \toprule
\texttt{10\_Protoss} & \texttt{10\_Protoss\_Hard} \\
        \midrule
\textbf{Training} & \textbf{Training} \\
 1 stalker \& 9 zealots                        & 1 stalker \& 9 zealots                 \\
 3 stalkers \& 7 zealots                       & 2 stalkers \& 8 zealots                \\
 4 stalkers \& 6 zealots                       & 3 stalkers \& 7 zealots                \\
 5 stalkers \& 5 zealots                       & 4 stalkers \& 6 zealots                \\
 6 stalkers \& 4 zealots                       & 5 stalkers \& 5 zealots                \\
 8 stalkers \& 2 zealots                       & 6 stalkers \& 4 zealots                \\
 9 stalkers \& 1 zealot                        & 7 stalkers \& 3 zealots                \\
 4 zealots \& 6 colossi                        & 8 stalkers \& 2 zealots                \\
 5 zealots \& 5 colossi                        & 9 stalkers \& 1 zealot                 \\
 7 zealots \& 3 colossi                        & 4 zealots \& 6 colossi                 \\
 8 zealots \& 2 colossi                        & 5 zealots \& 5 colossi                 \\
 9 zealots \& 1 colossus                       & 6 zealots \& 4 colossi                 \\
 4 stalkers \& 6 colossi                       & 7 zealots \& 3 colossi                 \\
 5 stalkers \& 5 colossi                       & 8 zealots \& 2 colossi                 \\
 7 stalkers \& 3 colossi                       & 9 zealots \& 1 colossus                \\
 8 stalkers \& 2 colossi                       & 4 stalkers \& 6 colossi                \\
 10 stalkers                                   & 5 stalkers \& 5 colossi                \\
 10 zealots                                    & 6 stalkers \& 4 colossi                \\
 10 colossi                                    & 7 stalkers \& 3 colossi                \\
 8 stalkers \& 1 zealot \& 1 colossus          & 8 stalkers \& 2 colossi                \\
 1 stalker \& 8 zealots \& 1 colossus          & 9 stalkers \& 1 colossus               \\
 2 stalkers \& 7 zealots \& 1 colossus         & \textbf{Testing}                        \\
 6 stalkers \& 2 zealots \& 2 colossi          & 10 stalkers                             \\
 5 stalkers \& 3 zealots \& 2 colossi          & 10 zealots                              \\
 3 stalkers \& 5 zealots \& 2 colossi          & 10 colossi                              \\
 4 stalkers \& 4 zealots \& 2 colossi          & 8 stalkers \& 1 zealot \& 1 colossus  \\
 4 stalkers \& 3 zealots \& 3 colossi          & 1 stalker \& 8 zealots \& 1 colossus  \\
\textbf{Testing}                               & 7 stalkers \& 2 zealots \& 1 colossus \\
2 stalkers \& 8 zealots                        & 2 stalkers \& 7 zealots \& 1 colossus \\
7 stalkers \& 3 zealots                        & 6 stalkers \& 2 zealots \& 2 colossi  \\
6 zealots \& 4 colossi                         & 2 stalkers \& 6 zealots \& 2 colossi  \\
6 stalkers \& 4 colossi                        & 5 stalkers \& 3 zealots \& 2 colossi  \\
9 stalkers \& 1 colossus                       & 3 stalkers \& 5 zealots \& 2 colossi  \\
7 stalkers \& 2 zealots \& 1 colossus          & 4 stalkers \& 4 zealots \& 2 colossi  \\
3 stalkers \& 4 zealots \& 3 colossi           & 4 stalkers \& 3 zealots \& 3 colossi  \\
2 stalkers \& 6 zealots \& 2 colossi           & 3 stalkers \& 4 zealots \& 3 colossi  \\
        \bottomrule
    \end{tabular}
\end{table}

\subsection{Architecture, Training and Evaluation}\label{app:architecture}

The evaluation procedure is similar to the one in \citep{rashid2020monotonic}. The training is paused after every 30k timesteps during which 16 test episodes are run with agents performing action selection greedily in a decentralised fashion. The percentage of episodes where the agents defeat all enemy units within the permitted time limit is referred to as the
test win rate.

To speed up the learning, the agent networks are parameters are shared across all agents. 
A one-hot encoding of the \texttt{agent\_id} is concatenated onto each agent's observations. 
All neural networks are trained using RMSprop without weight decay or momentum.

\subsubsection*{Value-based baselines}

The architecture of all agent networks is a DRQN \cite{hausknecht_deep_2015} with a recurrent layer 
comprised of a GRU with a 64-dimensional hidden state, with a fully-connected 
layer before and after. 
We sample batches of 32 episodes uniformly from the replay buffer, and train on fully unrolled episodes, performing a single gradient descent step after 8 episodes.

\begin{table}[H]
    \centering
    \caption{Hyperparameters of QMIX and VDN}\label{tab:hyper_value_based}
    \begin{tabular}{llc}
        \toprule
        Method & Name & Value \\
        \midrule
        QMIX \& VDN & learning rate & $5 \times 10^{-4}$\\
        & RMSprop $\alpha$ & 0.99 \\
        & replay buffer size & 5000 episodes \\
        & target network update interval & 200 episodes \\
        & $\gamma$ & 0.99 \\
        & double DQN target & True \\
        & initial $\epsilon$ & 1 \\
        & final $\epsilon$ & 0.05 \\
        & $\epsilon$ anneal period & 50000 steps \\
        & $\epsilon$ anneal rule & linear \\
        \midrule
        QMIX & mixing network hidden layers & 1 \\
        & mixing network hidden layer units & 32 \\
        & mixing network non-linearity & ELU \\
        & hypernetwork hidden layers & 2 \\
        & hypernetwork hidden layer units & 64 \\
        & hypernetwork non-linearity & ReLU \\
        \bottomrule
    \end{tabular}
\end{table}

\subsubsection*{PPO baselines}
We parameterize the actor and critic with two independent recurrent neural networks, 
each of which is  comprised of a GRU with a 64-dimensional hidden state, with a fully-connected
layer as the input and output. 

\begin{table}[H]
    \centering
    \caption{Hyperparameters of IPPO and MAPPO}\label{tab:hyper_value_based}
    \begin{tabular}{llc}
        \toprule
        Method & Name & Value \\
        \midrule
        IPPO \& MAPPO & critic learning rate & $0.001$ \\
        & actor learning rate & 0.99 \\
        & $\gamma$ & 0.99 \\
        & $\lambda$ & 0.95 \\
        & $\epsilon$ & 0.2 \\
        & clip range & 0.1 \\
        & normalize advantage & True \\
        & normalize inputs & True \\
        & grad norm & 0.5 \\
        & number of actors & 8 \\
        & critic coefficient & 2 \\
        & entropy coefficient & 0 \\
        & mini epochs for actor update & 10 \\
        & mini epochs for critic update & 10 \\
        & mini batch size & 64 \\
        \bottomrule
    \end{tabular}
\end{table}
	
\section{Full StarCraft II Results}\label{app:results}
Complete results for StarCraft II are as shown in \cref{fig:smac1},\cref{fig:smac2}, \cref{fig:smac3}.
\begin{figure*}[htb!]
    \centering
    \includegraphics[width=.43\linewidth]{figures/plots/sc2_baseline_legend.png}
    \includegraphics[width=\linewidth]{figures/plots/10gen_win_rate.pdf}
    \includegraphics[width=\linewidth]{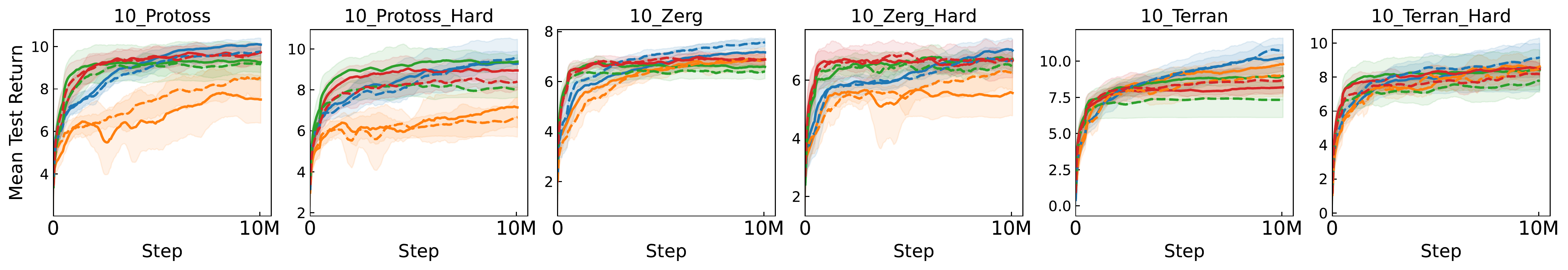}
    \includegraphics[width=\linewidth]{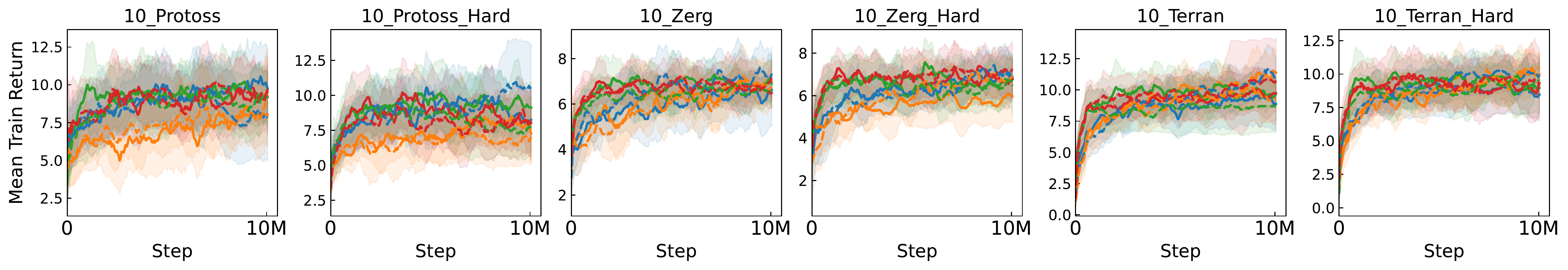}
    \includegraphics[width=\linewidth]{figures/plots/10gen_gap.pdf}
    \caption{Experimental results on SMAC unit swapping tasks. Dashed lines indicate the inclusion of information on capabilities as part of the agent observations. Standard deviation is shaded.}
    \label{fig:smac1}
\end{figure*}

\begin{figure*}[htb!]
    \centering
    \includegraphics[width=.43\linewidth]{figures/plots/sc2_baseline_legend.png}
    \includegraphics[width=\linewidth]{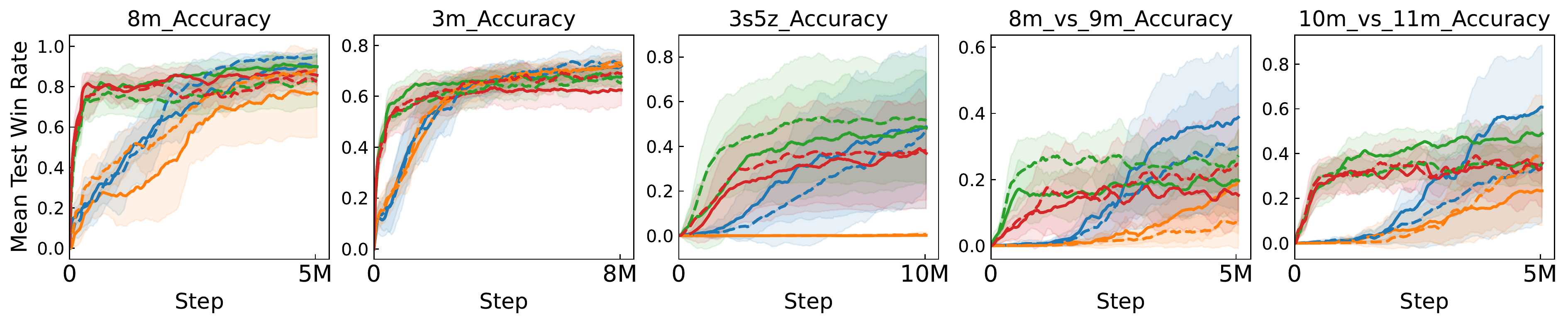}
    \includegraphics[width=\linewidth]{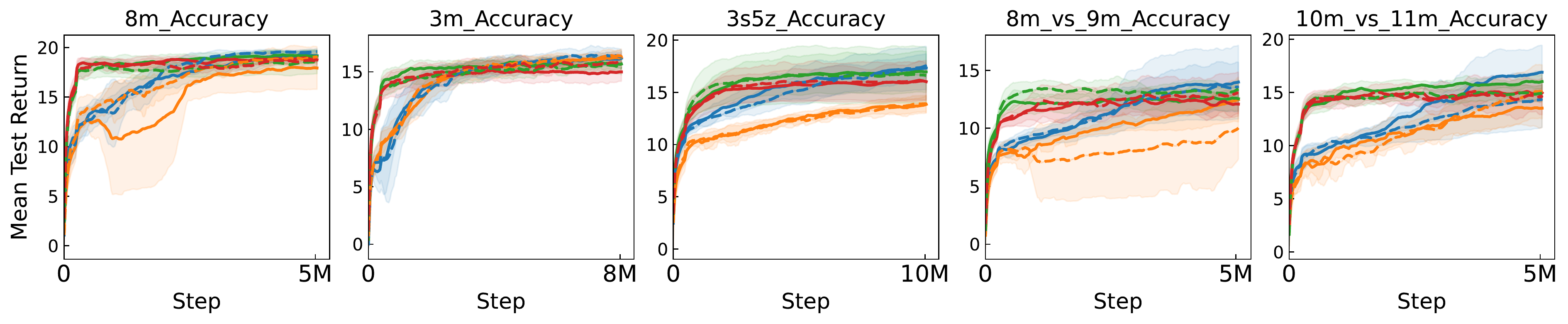}
    \includegraphics[width=\linewidth]{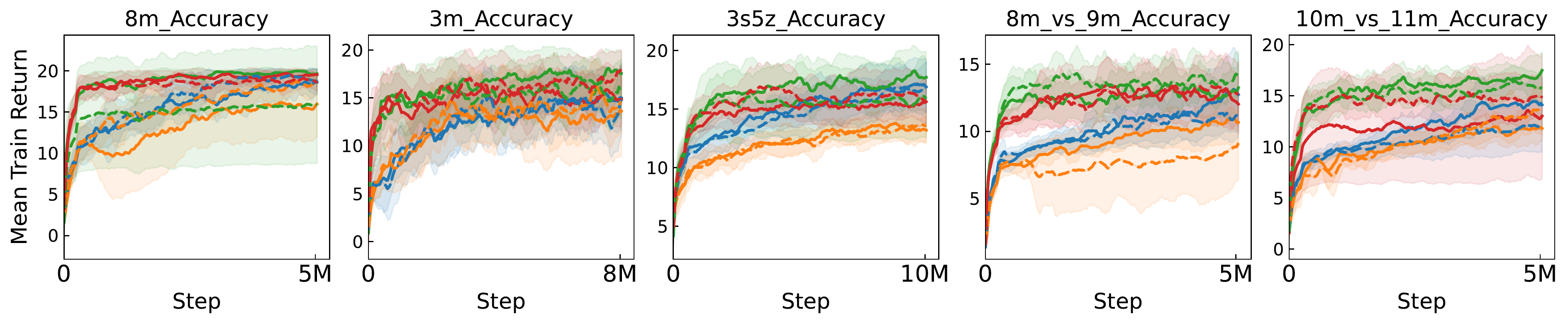}
    \includegraphics[width=\linewidth]{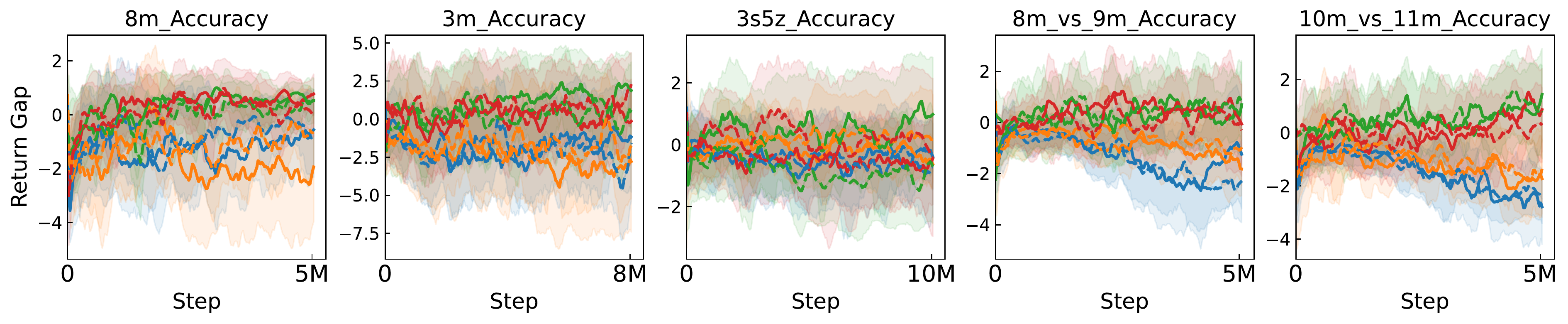}
    \caption{Experimental results on SMAC unit accuracy tasks. Dashed lines indicate the inclusion of information on capabilities as part of the agent observations. Standard deviation is shaded.}
    \label{fig:smac2}
\end{figure*}

\begin{figure*}[htb!]
    \centering
    \includegraphics[width=.43\linewidth]{figures/plots/sc2_baseline_legend.png}
    \includegraphics[width=\linewidth]{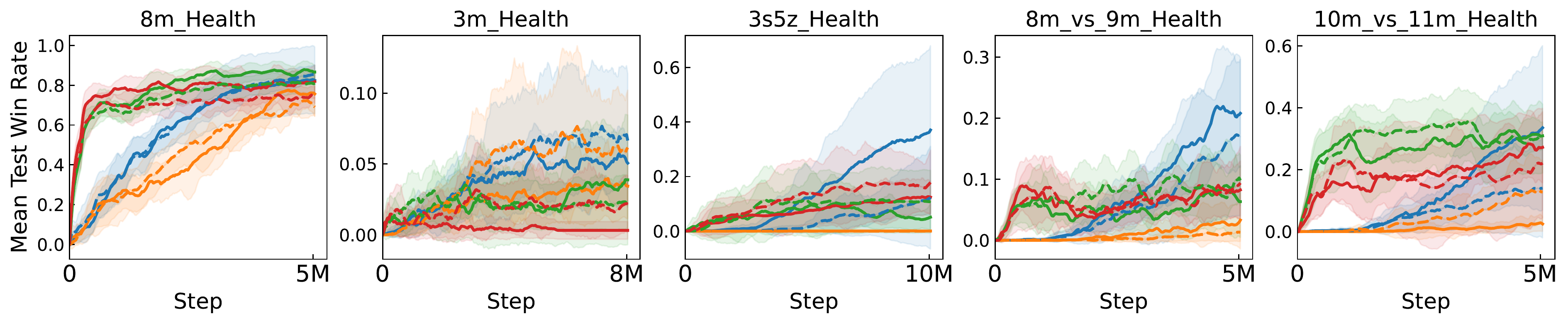}
    \includegraphics[width=\linewidth]{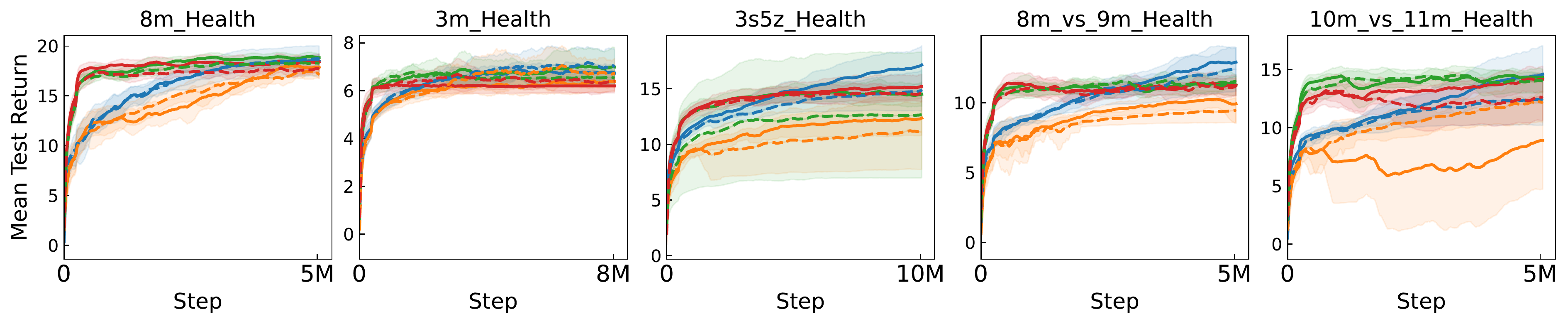}
    \includegraphics[width=\linewidth]{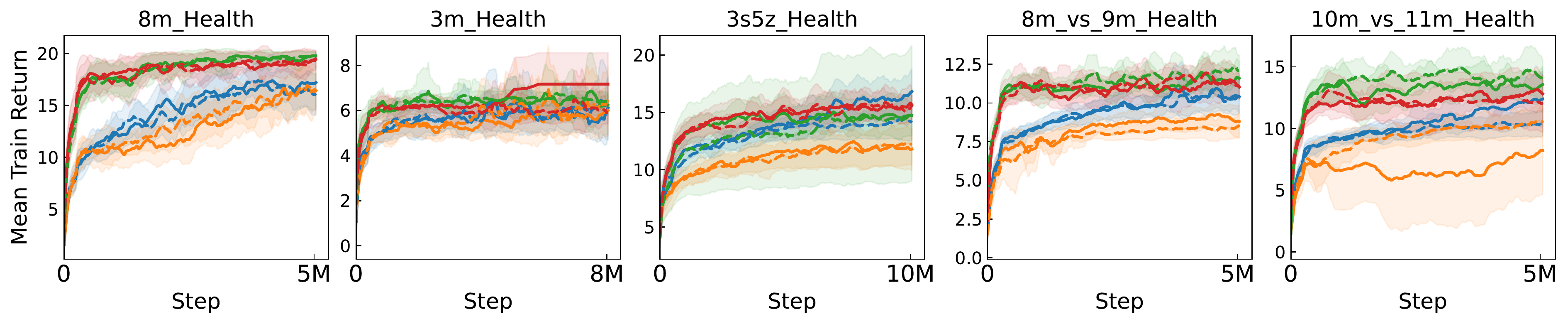}
    \includegraphics[width=\linewidth]{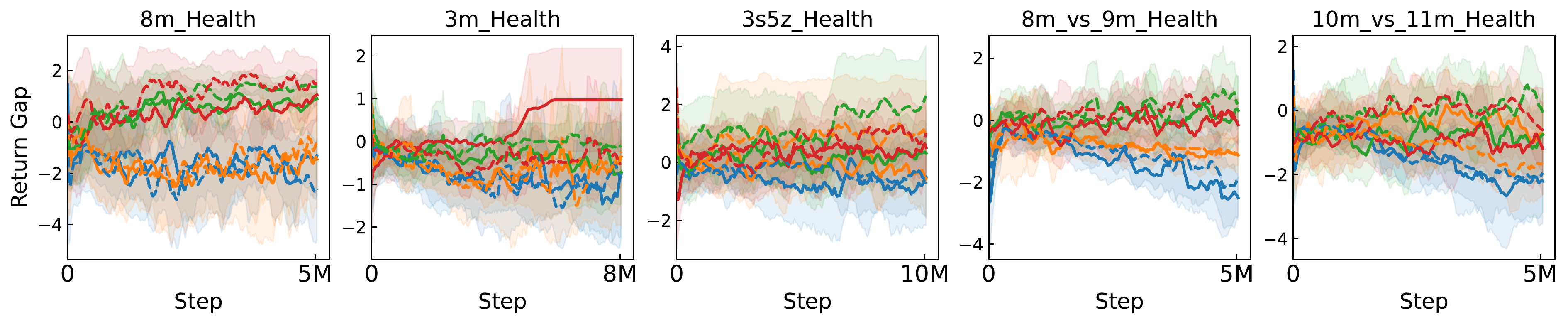}
    \caption{Experimental results on SMAC unit health tasks. Dashed lines indicate the inclusion of information on capabilities as part of the agent observations. Standard deviation is shaded.}
    \label{fig:smac3}
\end{figure*}	
\typeout{get arXiv to do 4 passes: Label(s) may have changed. Rerun}
\end{document}